\documentclass[10pt,twocolumn,letterpaper]{article}

\usepackage[table,dvipsnames]{xcolor}
\usepackage[pagenumbers]{cvpr} % To force page numbers, e.g. for an arXiv version

\usepackage{graphicx}
\usepackage{amsmath}
\usepackage{amssymb}
\usepackage{booktabs}
\usepackage{xspace}

\newcommand{\alias}{DexHandDiff\xspace}
\newcommand{\myparagraph}[1]{\vspace{3pt}\noindent\textbf{#1}}

\definecolor{bestcolor}{gray}{.9}
\newcommand{\bestcell}[1]{\cellcolor{bestcolor}{#1}}

\newcommand{\bs}{\boldsymbol{s}}
\newcommand{\ba}{\boldsymbol{a}}

\newcommand{\btau}{\boldsymbol{\tau}}
\newcommand{\by}{\boldsymbol{y}}

\usepackage{changepage}

\usepackage{amsmath}
\usepackage{amssymb}
\usepackage{mathtools}
\usepackage{pifont}
\usepackage{amsthm}
\usepackage{makecell}
\usepackage{multirow}
\usepackage{tcolorbox}          % 主要的框体包
\usepackage{listings}           % 代码列表包
\tcbuselibrary{listings}        % 使用tcolorbox的listings库
\tcbuselibrary{breakable}       % 允许跨页断行

\newtcblisting{promptcode}[1][Python Prompt Example]{
    listing only,
    listing options={
        breaklines=true,       % 允许自动换行
        breakatwhitespace=true,% 在空格处换行
        keepspaces=true,       % 保持空格
        basicstyle=\ttfamily\scriptsize, % 基本字体样式
        showstringspaces=false,% 不显示字符串中的空格
        tabsize=4,            % 制表符宽度
        columns=flexible,      % 灵活列宽
        escapeinside={/*}{*/},  % 添加这一行
    },
    breakable=true,           % 允许跨页
    width=\linewidth,         % 宽度为行宽
}

\newtcblisting{pythoncode}[1][Python Prompt Example]{
    listing only,
    listing options={
        language=Python,
        breaklines=true,       % 允许自动换行
        breakatwhitespace=true,% 在空格处换行
        keepspaces=true,       % 保持空格
        basicstyle=\ttfamily\scriptsize, % 基本字体样式
        showstringspaces=false,% 不显示字符串中的空格
        tabsize=4,            % 制表符宽度
        columns=flexible,      % 灵活列宽
        escapeinside={/*}{*/},  % 添加这一行
    },
    breakable=true,           % 允许跨页
    width=\linewidth,         % 宽度为行宽
}

\theoremstyle{plain}
\newtheorem{theorem}{Theorem}[section]

\theoremstyle{definition}

\theoremstyle{remark}

\usepackage{algorithm}
\usepackage{algpseudocode}
\algrenewcommand\algorithmicrequire{\textbf{Input:}}
\algrenewcommand\algorithmicensure{\textbf{Output:}}

\definecolor{cvprblue}{rgb}{0.21,0.49,0.74}
\usepackage[pagebackref,breaklinks,colorlinks,allcolors=cvprblue]{hyperref}

\usepackage[accsupp]{axessibility}  % Improves PDF readability for those with disabilities.

\hypersetup{linkcolor={red}}  

\def\paperID{7589} % *** Enter the Paper ID here
\def\confName{CVPR}
\def\confYear{2025}

\title{DexHandDiff: Interaction-aware Diffusion Planning for\\Adaptive Dexterous Manipulation}

\author{
Zhixuan Liang$^{1,2}$\footnotemark[2] \quad
Yao Mu$^1$ \quad 
Yixiao Wang$^2$ \quad
Tianxing Chen$^{1}$ \quad
Wenqi Shao$^1$\\
Wei Zhan$^2$ \quad
Masayoshi Tomizuka$^2$\footnotemark[3] \quad 
Ping Luo$^{1}$\footnotemark[3] \quad
Mingyu Ding$^{2}$\\
[1.5mm]
$^1$The University of Hong Kong \quad
$^2$University of California, Berkeley\\
{\tt\small \{zxliang, ymu, pluo\}@cs.hku.hk \quad \{yixiao\_wang, wzhan, tomizuka, myding\}@berkeley.edu}
\\
\normalsize{\url{https://dexdiffuser.github.io/}}
}

\begin{document}
\maketitle
\footnotetext[2]{This work was done during Zhixuan's visit to UC Berkeley.}
\footnotetext[3]{Corresponding authors.}

\begin{abstract}
Dexterous manipulation with contact-rich interactions is crucial for advanced robotics. 
% Dexterous manipulation is crucial for advanced robotics, especially those with contact-rich interactions between robot hands and objects. 
While recent diffusion-based planning approaches show promise for simple manipulation tasks, they often produce unrealistic ghost states (e.g., the object automatically moves without hand contact) or lack adaptability when handling complex sequential interactions.
% Recent diffusion-based approaches show promise for simpler manipulation tasks but fall short in modeling interaction-aware physics and handling flexible goals in complex processes.
In this work, we introduce \alias, an interaction-aware diffusion planning framework for adaptive dexterous manipulation.
\alias models joint state-action dynamics through a dual-phase diffusion process which consists of pre-interaction contact alignment and post-contact goal-directed control, enabling goal-adaptive generalizable dexterous manipulation.
Additionally, we incorporate dynamics model-based dual guidance and leverage large language models for automated guidance function generation, enhancing generalizability for physical interactions and facilitating diverse goal adaptation through language cues.
% \alias integrates text-to-guidance to facilitate diverse goal adaptation through language cues. 
Experiments on physical interaction tasks such as door opening, pen and block re-orientation, object relocation, and hammer striking demonstrate \alias's effectiveness on goals outside training distributions, achieving over twice the average success rate (59.2\% vs.~29.5\%) compared to existing methods. 
%
% Our framework achieves 70.0\% success on 30-degree door opening, 40.0\% and 36.7\% on pen and block half-side re-orientation respectively, 93.3\% on half-side relocation and 46.7\% on hammer nail half drive, highlighting its robustness and flexibility in contact-rich manipulation.
Our framework achieves an average of 70.7\% success rate on goal adaptive dexterous tasks, highlighting its robustness and flexibility in contact-rich manipulation.

\end{abstract}    
\section{Introduction}
\label{sec:intro}

% However, RL approaches often face limitations in generalizing to new or more dynamic tasks.
% %
% Tasks such as door opening, tool handling, and object reorientation require precision and adaptability, often involving multi-contact interactions and dynamically evolving goals.
% %
% Traditional control policies, especially in dexterous hand manipulation, have generally relied on reinforcement learning (RL)~\cite{}, effective for narrowly defined tasks but heavily rely on tailored reward functions \my{we also need human-designed reward functions?} crafted for each task and constraint. This dependency limits their generalizability to new applications or tasks with flexible interaction requirements, where adaptability is crucial.-=
% %
% However, new data must continually be collected for novel tasks with unique dynamics, limiting these methods in adapting to complex and contact-rich interactions.

\begin{figure}[tb]
  \centering
   \includegraphics[width=1.0\linewidth]{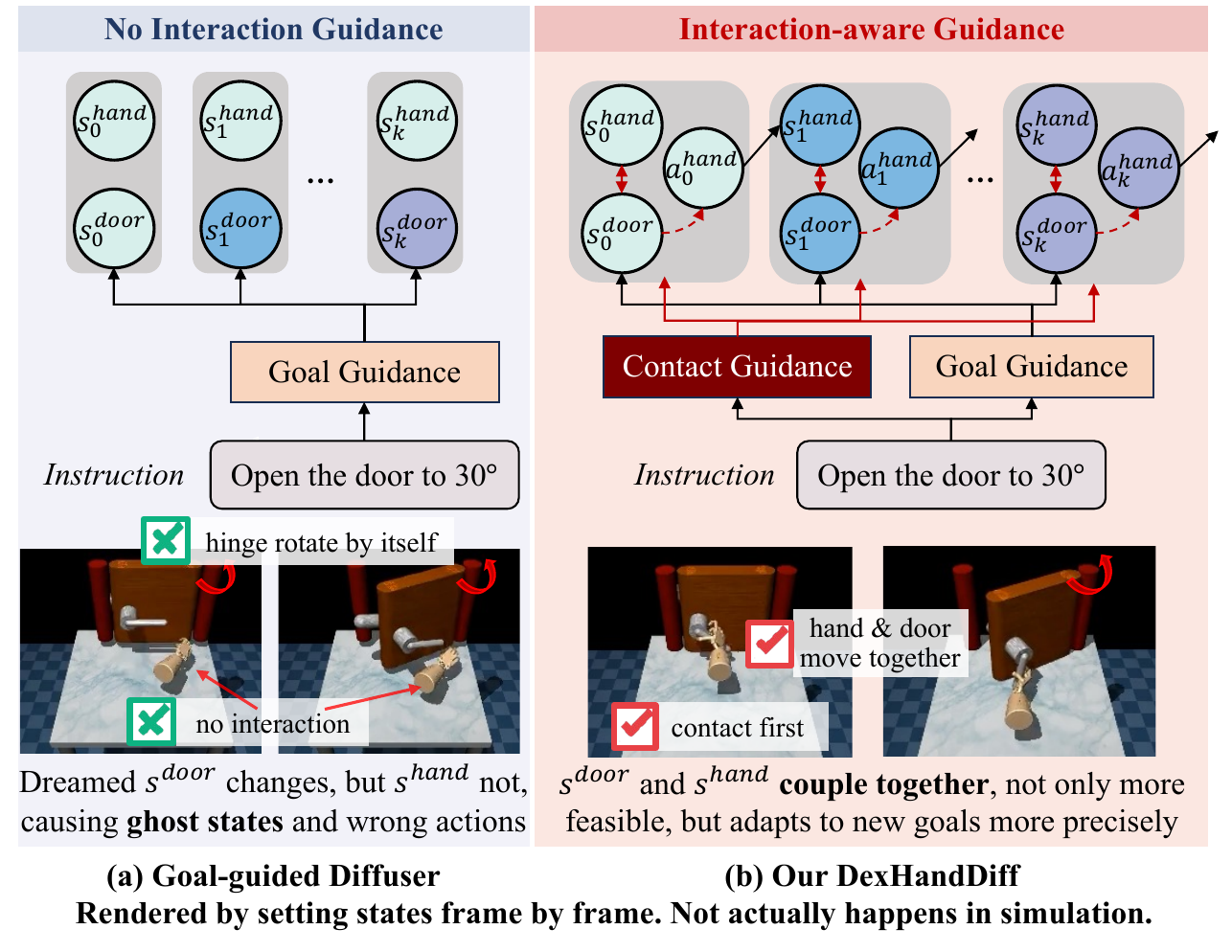}
   \includegraphics[width=1.0\linewidth]{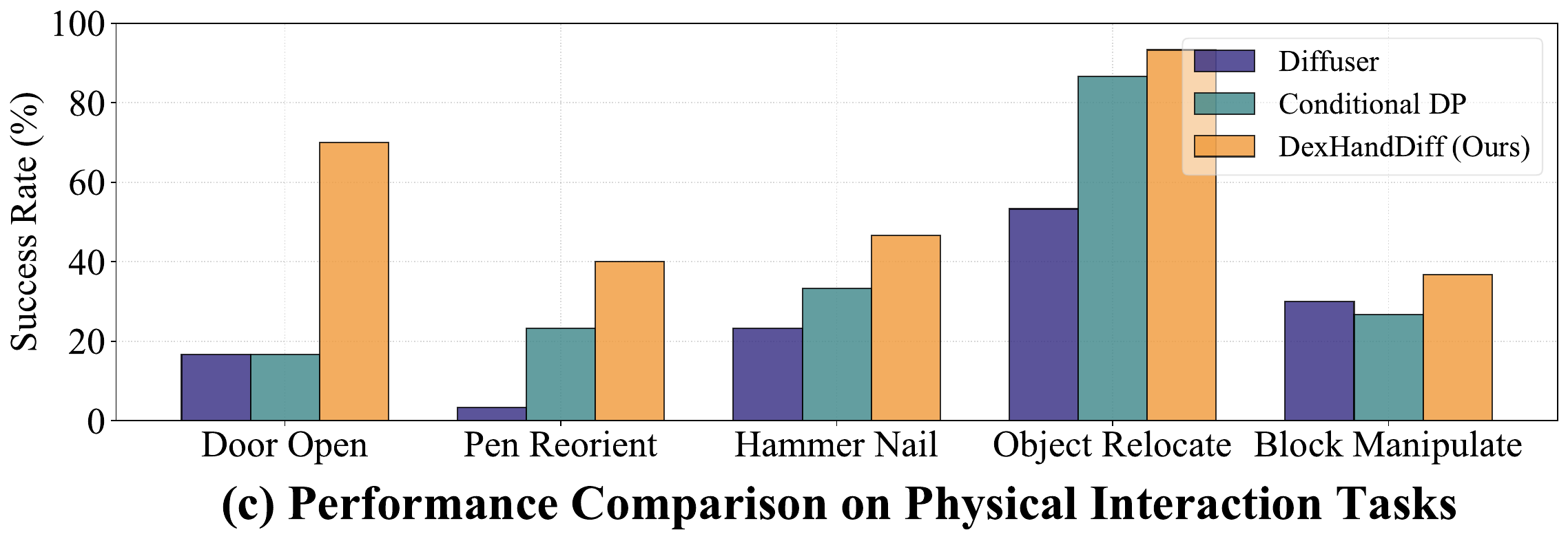}
   \vspace{-20pt}
   \caption{
   % \textbf{Comparison of \alias with traditional goal-guided diffusion approaches.} 
   (a) Previous diffusers directly apply goal guidance to object states, which causes ghost states, where objects appear to move independently without hand contact-a physically impossible scenario. (b) \alias introduces contact guidance that jointly influences both hand/object states and hand actions, while maintaining tight state-action coupling. It prevents ghost states, and enables precise goal adaptation. (c) Quantitative comparisons with previous methods on goal-shifted interaction tasks.}
   \label{fig:teaser}
   \vspace{-16pt}
\end{figure}

Dexterous manipulation, a cornerstone of advanced robotics with applications from service robotics to industrial automation, remains a challenging problem despite advances in reinforcement learning (RL)~\cite{openai2020learning,akkaya2019solving,chen2022system,wan2023unidexgrasp++,wu2024unidexfpm} and imitation learning~\cite{mandikal2022dexvip,huang2021generalization}.
Recently, diffusion-based planning~\cite{diffuser,decisiondiffuser,adaptdiffuser,diffusion_policy} has emerged as a promising new representative of imitation learning, capable of learning intricate motion trajectories from demonstration data for smoother and more adaptable control.
However, current diffusion approaches are primarily designed for simpler gripper-based (one Degree of Freedom) manipulation tasks, focusing on either trajectory completion or action replay by reaching target positions sequentially.
They fall short in dexterous hand manipulation requiring part-aware precise interaction and exhibiting rich contact dynamics through multi-finger control and in-hand adjustment.
% capturing the staged and contact-rich interactions required for more sophisticated tasks, such as door opening and tool handling, which involve dexterous multi-fingered robotic hands.

More specifically, 
% Current diffusion-based planners can be generally divided into two streams based on whether they generate actions or states.
%
existing diffusion on action models~\cite{diffusion_policy,ACT} (\ie models generating actions) excel in well-defined tasks but often lack generalizability in adapting to complex or new tasks with flexible interaction requirements. They necessitate continual data collection for new goal configurations even within the same dynamics, limiting their effectiveness in contact-rich interactions.
In contrast, diffusion on state methods~\cite{diffuser,decisiondiffuser,metadiffuser}, including those adapted from video diffusion models for imitation learning~\cite{unipi,GR-2}, will produce unrealistic ``ghost states'' in interaction tasks.
As shown in Fig.~\ref{fig:teaser} and Fig.~\ref{fig:ghost_state}, the visualizations are rendered by setting states frame by frame with predicted output from state-based methods, and show objects react independently of physical contact (\eg drawers opening on their own before the manipulator reaches them), which cannot actually happen and would result in failure.
This issue arises because the object states can't be directly controlled. Actions must first influence dexterous hand's states before impacting the object, revealing the importance of modeling state transitions for physics-driven interactions.
%
% Addressing these limitations in contact-rich dexterous manipulation requires a model that is both interaction-aware and adaptive to task constraints, while remaining grounded in realistic physical behavior.

Thus, we propose~\alias, an interaction-aware diffusion model tailored for adaptive dexterous manipulation that exhibits goal shifts or cost function variations while maintaining similar dynamics.
\alias models joint state-action dynamics that takes the state output to guide and constrain the action output with realistic physical behavior.
A dynamics model-based dual guide is incorporated to maintain coherence with dynamics observed in training data. It addresses the action-state consistency challenge first identified in Diffuser~\cite{diffuser} which however prioritized generated state over action, as shown in Fig.~\ref{fig:teaser}.

Specifically, \alias adopts a goal-adaptive diffusion mechanism with dual-phase process. 
1) At first, pre-contact phase, it guides the manipulator to align with the object’s key contact point, such as a handle or the center of object, ensuring stable alignment before initiating physical interaction. 
2) In the subsequent post-contact phase, it introduces joint guidance over both the manipulator and the object states, enabling fine-grained control to achieve the target state for the object.
This sequential approach integrates both action diffusion that prevents premature influence on the object's state before contact, and state diffusion that ensures effective goal alignment throughout.
By generating states and actions in an interaction-aware manner, \alias produces more coherent and realistic trajectories suited to complex tasks like tool using. 
Furthermore, to automate guidance function design, \alias introduces an approach using large language models in the text-to-reward paradigm, that can generalize across diverse goals and cost functions via language cues.
% in a classifier-guided structure.

We conduct experiments on multiple dexterous manipulation tasks to evaluate \alias's effectiveness, 
% including door opening, pen re-orientation, and hammer striking, 
covering both in-domain and goal-adaptability challenges, \eg, adapting to new goal ``door closing'' from ``90-degree door opening'' training data.
%
% Our method showed strong performance on tasks requiring precise sequential control, such as adapting to new goals (30, 50, 70, and 110 degrees or even reversal task) in the door environment, and succeeded on challenging, out-of-distribution goals for pen re-orientation.
%
Results with up to 70.0\% success rate on the 30-degree door task (vs.~the next best 16.7\% for Diffusion Policy) and 46.7\% on the hammer nail half-drive task (vs.~the next best 33.3\% for Decision Diffuser), confirm \alias's robustness and adaptability in capturing complex hand-object-environment interactions.

In summary, \alias advances adaptive dexterous manipulation by:
1) We propose the first interaction-aware, goal-adaptive diffusion planner for dexterous manipulation, modeling manipulator-object-environment dependencies to handle sequential tasks with complex state transitions.
2) By jointly modeling state-action behaviors with dynamics-based dual guidance and LLM-based interaction guidance, \alias sets a new standard for adaptive planning in dexterous manipulation and for the first time extends text-to-reward concepts to diffusers.
% , unlocking new possibilities for language goal-driven robotic systems.
%
3) Experimental validation on diverse dexterous manipulation tasks, demonstrating its robustness and adaptability. \alias achieves over twice the average success rate of the next best method (59.2\% vs.~29.5\%) across goal-directed tasks. 

% with notable gains such as 36.7\% on half-side pen re-orientation (vs. 3.3\% for Diffuser) and 46.7\% on nail half drive (vs. 33.3\% for Decision Diffuser).

% \alias sets a new standard for adaptive, interaction-aware planning in dexterous manipulation, unlocking new possibilities for goal-driven robotic systems across complex, variable real-world applications.

% By employing a classifier-guided structure, \alias allows policies to interpret nuanced task requirements via language cues

% \alias enables more flexible and explicit definitions of guidance, expanding diffusion-based imitation learning to accommodate diverse and complex interaction requirements. 

\section{Related Works}
\label{sec:related_work}

\myparagraph{Dexterous Manipulation.}
Dexterous manipulation~\cite{rajeswaran2017learning,gupta2021reset,sivakumar2022robotic,qin2022dexmv,chen2022dextransfer,chen2024vividex,luo2024grasping,he2024omnih2o,wang2024dexcap,weng2024dexdiffuser}
with multi-fingered hands enables complex tasks in unstructured environments by mimicking human hand flexibility. Initially, traditional methods using trajectory optimization and precise dynamics models~\cite{adroit,nagabandi2020deep}, struggled with high-dimensional action spaces and contact-rich dynamics. 
This led to the adoption of reinforcement learning (RL)~\cite{adroit,chen2022towards,RLdexterousReview,wu2024unidexfpm} for handling complex, high-DOF (degree of freedom) interactions. However, RL requires extensive online exploration and carefully designed reward functions~\cite{nagabandi2020deep,bidexterous} where inadequate reward shaping can hardly learn and it limits adaptability~\cite{yu2022dexterous,zhou2024learning}. Demonstration-based methods~\cite{zhou2024learning} reduce sample complexity, but they struggle to generalize across sequential, contact-rich tasks. Our \alias addresses these challenges by explicitly modeling hand-object-environment interactions, enabling goal-adaptive planning without intricate reward shaping, thus allowing for more efficient learning in complex, sequential tasks.

\myparagraph{Diffusion-based Planning Methods.}
% \todo{Change Topic to Imitation Learning}
Planning with diffusion models has become prominent in imitation learning for robotic manipulation~\cite{diffuser,diffusion_policy,metadiffuser,adaptdiffuser,skilldiffuser,chen2024g3flow,decisiondiffuser,liang2024dexdiffuser}. 
% with both classifier-guided and classifier-free variants widely adopted.
Classifier-guided methods~\cite{diffuser,adaptdiffuser} used task-specific classifiers to condition policies, while classifier-free ones integrated task variations within diffusion model~\cite{decisiondiffuser}. However, classifier-free methods lack flexibility for zero-shot explicit conditioning tasks due to reliance on training data configurations.
% whereas classifier-guided diffusion, as used in \alias, retains adaptability by providing gradient-based goal condition, making it suitable for complex tasks.
%
% Diffusion models can operate on either state or action spaces. Action diffusion is commonly used for preciser action sequences without inverse kinematics errors but lacks explicit state guidance for adaptable manipulation.
% % limiting adaptable goal setting at specific start or end states. 
% State-based diffusion, while enabling precise goal specification, is primarily limited to environments where all degrees of freedom are directly controllable (\eg, Mujoco Half-Cheetah, Hopper environments~\cite{mujoco} or straightforward pick-and-place tasks~\cite{diffuser}) and faces challenges with inverse dynamics errors in contact-rich manipulation.
%
\alias addresses this by performing classifier-guided diffusion over both state and action spaces, enabling  precise interaction and rich-contact dynamics planning for more realistic, complex and adaptable manipulation.

\myparagraph{LLM-based Robot Policy Code Generation.}
Recent works~\cite{liang2023code,robocodex,chen2024roboscript,wu2024plot2code} have demonstrated the potential of LLMs in generating executable code for robotics tasks. Code as Policies~\cite{liang2023code} showed LLMs can effectively translate high-level task descriptions into functional robot control programs. Eureka~\cite{eureka} and Text2Reward~\cite{text2reward} further advanced this direction by generating crucial parameters or complete reward functions from language descriptions, demonstrating well-structured prompts with comprehensive environment information can enable reliable reward shaping.
Our work extends this text-to-code paradigm to imitation learning through diffusers. \alias provides a natural interface for LLM code generation through its guidance function formulation, bridging the gap between task specification and behavioral policies to learn.

% Our work extends this text-to-code paradigm to imitation learning through diffusion-based planners, bridging the gap between natural language task specification and learned policies in manipulation tasks.

\section{Preliminary}
\subsection{Diffusion Model as Policy}
We formulate the dexterous manipulation planning problem within the Markov Decision Process (MDP) framework~\cite{puterman1994markov}, defined as $\mathcal{M}=(\mathcal{S}, \mathcal{A}, \mathcal{T}, \mathcal{R}, \gamma)$. The objective is to find an optimal action sequence $\ba_{0:T}^*$ that satisfies:
% maximizes the expected cumulative rewards over horizon $T$:
\vspace{-8pt}
\begin{equation}
\ba_{0:T}^*=\underset{\ba_{0:T}}{\arg \max } \mathcal{J}(\bs_0, \ba_{0:T})=\underset{\ba_{0:T}}{\arg \max } \sum_{t=0}^T \gamma^t R(\bs_t, \ba_t),
\vspace{-11pt}
\end{equation}
where state transitions follow $\bs_{t+1}=\mathcal{T}(\bs_t, \ba_t)$.

Following~\cite{diffuser,decisiondiffuser}, we leverage diffusion models to address this planning problem by treating state and action trajectories $\btau$ as sequential data. The reverse process of diffusion learns to denoise trajectories from a standard normal distribution through conditional probability $p_\theta(\btau^{i-1} \mid \btau^i)$. The model is trained to maximize the likelihood:
\vspace{-6pt}
\begin{equation}
\small
p_\theta\left(\btau^0\right)=\int p\left(\btau^N\right) \prod_{i=1}^N p_\theta\left(\btau^{i-1} \mid \btau^i\right) \mathrm{d} \btau^{1: N},
\vspace{-6pt}
\end{equation}
with the optimization objective inspired by ELBO,
% where $p\left(\btau^N\right)$ is a standard normal distribution and $\btau^{0}$ denotes original sequence data:
\vspace{-4pt}
\begin{equation}
% \small
\theta^*=\arg \min _\theta-\mathbb{E}_{\btau^0}\left[\log p_\theta\left(\btau^0\right)\right].
\label{eq:theta_optim}
\vspace{-6pt}
\end{equation}

For practical implementation, we adopt the simplified surrogate loss~\cite{ddpm} that focuses on predicting the noise term:
\vspace{-8pt}
\begin{equation}
\mathcal{L}_{\text{denoise}}(\theta) = \mathbb{E}_{i, \btau^{0} \sim q, \epsilon \sim \mathcal{N}}[||\epsilon - \epsilon_{\theta}(\btau^i, i)||^{2}].
\label{eq:loss_uncond}
% \vspace{-2pt}
\end{equation}

\subsection{Classifier-free Conditional Diffusion Policy}
To generate high-reward trajectories, classifier-free guidance~\cite{diffusion_beatgan} has been transferred from image to trajectory generation~\cite{decisiondiffuser}. This approach incorporates guidance signals $\boldsymbol{y}(\btau)$ directly in the noise prediction model by:
% that emphasize desired trajectory characteristics. The noise prediction is modified as:
\vspace{-5pt}
\begin{equation}
\hat{\epsilon} = \epsilon_\theta(\btau^i, \varnothing, i) + \omega(\epsilon_\theta(\btau^i, \boldsymbol{y}, i) - \epsilon_\theta(\btau^i, \varnothing, i)),
\label{eq:hat_epsilon}
\vspace{-5pt}
\end{equation}
where $\omega$ controls the guidance strength, and $\varnothing$ denotes the absence of conditioning. 
% the part equal to unconditional diffusion learned from Eq.~\ref{eq:loss_uncond}. The training objective becomes:
% \vspace{-5pt}
% \begin{equation}
% \mathcal{L}_{\mathit{free}}(\theta) = \mathbb{E}_{i, \btau, \epsilon}\left[\left|\epsilon - \epsilon_\theta\left(\btau^i, (1 - \beta) \boldsymbol{y}(\btau^i) + \beta \varnothing, i\right)\right|^2\right].
% \label{eq:loss_diff}
% \end{equation}
% where $\beta$ controls the condition dropout rate to balance sample diversity and context adherence. 
During sampling, trajectories are generated with the predicted modified noise $\hat{\epsilon}$.
% employing re-parameterization technique.
% by iteratively applying the noise prediction model 

\subsection{Classifier-guided Diffusion Policy}
\label{sec:c-guided formulation}
Different from classifier-free diffusion models that condition relying solely on implicit representations within the training data,
% This becomes particularly restrictive when adapting to novel task variations that, while related to the training scenarios, require novel goal specification. 
classifier-guided approach, enables direct reward or goal conditioning through gradient-based guidance.
% To formulate it, we re-write the conditional diffusion process:
% \vspace{-5pt}
% \begin{equation}
% \label{eq:diff_plan}
% q(\btau^{i} | \btau^{i-1}), \;\;\;\; p_{\theta}(\btau^{i-1}|\btau^{i}, \boldsymbol{y}(\btau)),
% \vspace{-2pt}
% \end{equation}
% where $\boldsymbol{y}(\btau)$ encodes trajectory-specific information such as return $\mathcal{J}(\btau^0)$ or constraints. 

% The optimization objective becomes:
% \vspace{-3pt}
% \begin{equation}
% \theta^*=\arg \min_\theta-\mathbb{E}{\btau^0}\left[\log p\theta(\btau^{0} | \boldsymbol{y}(\btau^0))\right]
% \label{eq:cond_gen_model}
% \end{equation}

For reward maximization, it introduces trajectory optimality $\mathcal{O}_{t}$ at timestep $t$, following a Bernoulli distribution where ${p(\mathcal{O}_t=1) = \exp (\gamma^{t} \mathcal{R}(\bs_t, \ba_t))}$. The diffusion process can be naturally extended to incorporate conditioning by sampling from perturbed distributions: 
\vspace{-7pt}
\begin{equation}
\label{eq:perturbed}
\tilde{p}_\theta(\btau{}) = p(\btau{} \mid \mathcal{O}_{1:T}=1) \propto p_{\theta}(\btau{}) p(\mathcal{O}_{1:T}=1 \mid \btau{})
\vspace{-6pt}
\end{equation}

Under Lipschitz conditions on $p(\mathcal{O}_{1:T} \mid \btau^{i})$~\cite{feller2015theory}, the reverse diffusion process follows:
\vspace{-6pt}
\begin{equation}
\label{eq:guided}
p_\theta(\btau^{i-1} \mid \btau^{i}, \mathcal{O}_{1:T}) \approx \mathcal{N}(\btau^{i-1}; \mu_{\theta} + \alpha\Sigma g, \Sigma),
\vspace{-6pt}
\end{equation}
where the guidance gradient $g$ is:
\vspace{-6pt}
\begin{equation}
\begin{aligned}
g &= \nabla_{\btau} \log p(\mathcal{O}_{1:T} \mid \btau) |_{\btau = \mu_{\theta}} \\
\vspace{-2pt}
&= \sum_{t=0}^{T} \gamma^{t} \nabla_{\bs_t,\ba_t} \mathcal{R}(\bs_t, \ba_t) |_{(\bs_t,\ba_t)=\mu_t}
= \nabla_{\btau} \mathcal{J}(\mu_{\theta}).
\end{aligned}
\vspace{-8pt}
\end{equation}

For discrete goal conditioned tasks, the constraint can be simplified by directly substituting conditional values at each diffusion timestep $i \in \{0, 1, ..., N\}$.

\begin{table*}[tb]
  \centering
  \small
  \setlength{\tabcolsep}{10pt}
  \resizebox{0.92\linewidth}{!}{
  \begin{tabular}{l|ccc|ccc}
    \toprule
    \textbf{Method} & \makecell{Diffusion on\\State or Action}& \makecell{Diffusion\\Condition Type}& \makecell{Action Gen\\Method} & \makecell{Goal\\Adaptability} & \makecell{No Ghost\\States} & \makecell{Interaction\\Aware} \\
    \midrule
    \textbf{Diffuser}~\cite{diffuser} & State & Classifier-Guided & Inverse Dyn & \checkmark & $\times$ & $\times$ \\
    \textbf{Decision Diffuser}~\cite{decisiondiffuser} & State & Classifier-Free & Inverse Dyn & $\times$ (if diverse data, then \checkmark) & $\times$ & $\times$ \\
    \textbf{Diffusion Policy}~\cite{diffusion_policy} & Action & Classifier-Free & Direct & $\times$ (if diverse data, then \checkmark) & \checkmark & $\times$ \\
    \midrule
    \textbf{\alias (Ours)} & State \& Action & Classifier-Guided & Direct & \checkmark & \checkmark & \checkmark \\
    \bottomrule
  \end{tabular}}
  \vspace{-6pt}
  \caption{\textbf{Comparison of diffusion-based approaches for robot manipulation.} Quantitative results on door-opening are shown in Sec.~\ref{sec:experiment}.}
  \vspace{-12pt}
  \label{tab:analysis}
\end{table*}

\section{Analysis of Diffusion-based Planning Methods for Interaction-intensive Tasks}
\label{sec:challenge_of_existing}

Current diffusion-based methods are widely adopted for robotic manipulation but reveal significant limitations when applied to dexterous, sequential interaction tasks. Table~\ref{tab:analysis} provides an overview of prominent diffusion-based methods (including Diffuser~\cite{diffuser}, Decision Diffuser~\cite{decisiondiffuser}, Diffusion Policy~\cite{diffusion_policy} and our~\alias), categorizing each by their conditioning approach, action generation method, and goal adaptability.
In this section, we analyze these challenges across three key dimensions.
% explicit state conditioning, conditional adaptability, and state-transition coherence. 

% Additionally, we illustrate specific issues, such as the "ghost state" phenomenon, through video examples in Fig. 1.

\myparagraph{Action-only Diffusion is Limited in Explicit State Conditioning.}
Existing diffusion on action models like Diffusion Policy (DP)~\cite{diffusion_policy}, excel in providing precise, consistent action control, benefiting from extensive training data and bypassing errors from inverse kinematics.
% making them well-suited for robotic arm tasks. 
% Action diffusion ensures stable action precision despite variations in arm dynamics, and bypasses errors from inverse kinematics.
% by directly predicting actions rather than deriving them from states. 
They yield high performance when training data is sufficient and diverse. However, for tasks requiring variant multi-stage goals, action-only diffusion lacks the flexibility to perform explicit state guidance at intermediate stages, like aligning hand and object state at pre-grasp stage, hurting the adaptability of the whole planner. 
% As shown in Table~\ref{tab:analysis}, action-only diffusion methods struggle with adaptability; 
For example, DP trained on data with opening the door to 90 degrees hardly adapt to open 30 or 60 degrees.

\myparagraph{Ghost States in State-only Diffusion for Sequential Interaction.}
While state-based diffusion models offer the advantage of flexible goal specification, it is only effective in fully actuated tasks where all degrees of freedom (DoF) are directly controllable, such as MuJoCo~\cite{mujoco,diffuser}, and gripper pick-and-place (requiring only end-effector position control)~\cite{decisiondiffuser,diffusion_policy} tasks. In such scenarios, all states of the system can be manipulated directly.
However, in contact-rich interaction task where indirect control exists, such as striking a nail with a hammer using a dexterous hand, additional uncontrollable DoFs, like the hammer head and nail positions, must be changed through transitions from the states of the hand. Applying generation across all states, including those of objects beyond the hand, will result in unrealistic ``ghost states'' where objects appear to move independently of contact but actually cannot, as illustrated in Fig.~\ref{fig:teaser} and Fig.~\ref{fig:ghost_state}. 
% It disrupts the realism required for interaction tasks that depend on adaptive, contact-based control adjustments.

\myparagraph{Classifier-free vs.~Classifier-guided Adaptability.}
Classifier free diffusion models, valued for not requiring external classifiers, encode task variations directly within the model. This structure is effective for tasks with constrains in observed configurations, but with limited goal adaptability in zero-shot or new-task scenarios.
% where goals and conditions differ from training data. 
For instance, in the push-T task, DP cannot directly adapt to new target positions due to the fixed goal in training data. 
% which is similar to our door experiments with training data including only 90$^\circ$ target. 
In contrast, classifier-guided methods, such as ours, mitigate this limitation by offering adaptable, gradient-based guidance, enabling direct conditioning on new goals or rewards, enhancing flexibility across a range of tasks.

\section{Method}
\label{sec:method}
\subsection{Interaction-aware Diffusion-based Planning}
To address these limitations, we propose~\alias, an interaction-aware diffusion planning framework (Fig.~\ref{fig:framework}), maintaining physical consistency and enabling flexible goal adaptation for dexterous manipulation.

% To address the limitations of existing approaches, we propose~\alias, an interaction-aware diffusion planning framework specifically designed for adaptive dexterous manipulation. Our framework incorporates task-specific guidance strategies that vary based on the nature of the interaction and task requirements.

\begin{figure}[tb]
  \centering
   \includegraphics[clip, trim=0 2.5cm 0 7.5cm, width=\linewidth]{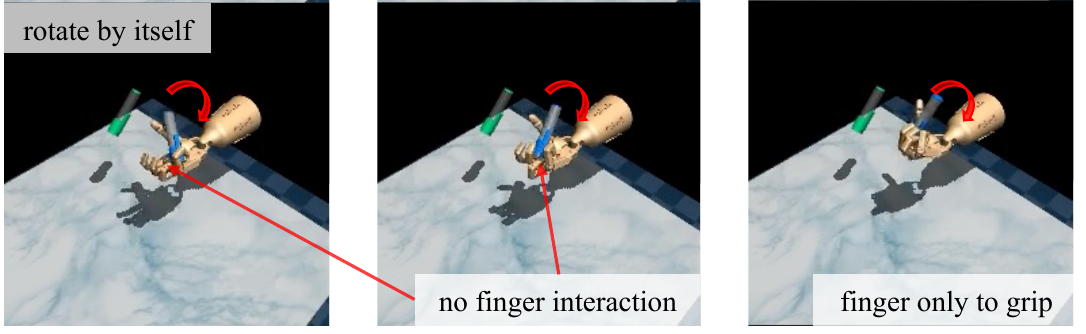}
   \vspace{-20pt}
   \caption{
   \textbf{Demonstration of ghost states on the pen reorientation task.} The visualizations are rendered by setting predicted states frame by frame, which cannot actually happen and will lead to failure.
% \textbf{Ghost state on pen reorientation.} 
The pen appears to autonomously rotate to the desired pose without any hand manipulation, and the fingers look like moving to grasp the pen at the last frame.
   }
   \label{fig:ghost_state}
   \vspace{-15pt}
\end{figure}

\begin{figure*}[tb]
  \centering
   \includegraphics[width=0.95\linewidth]{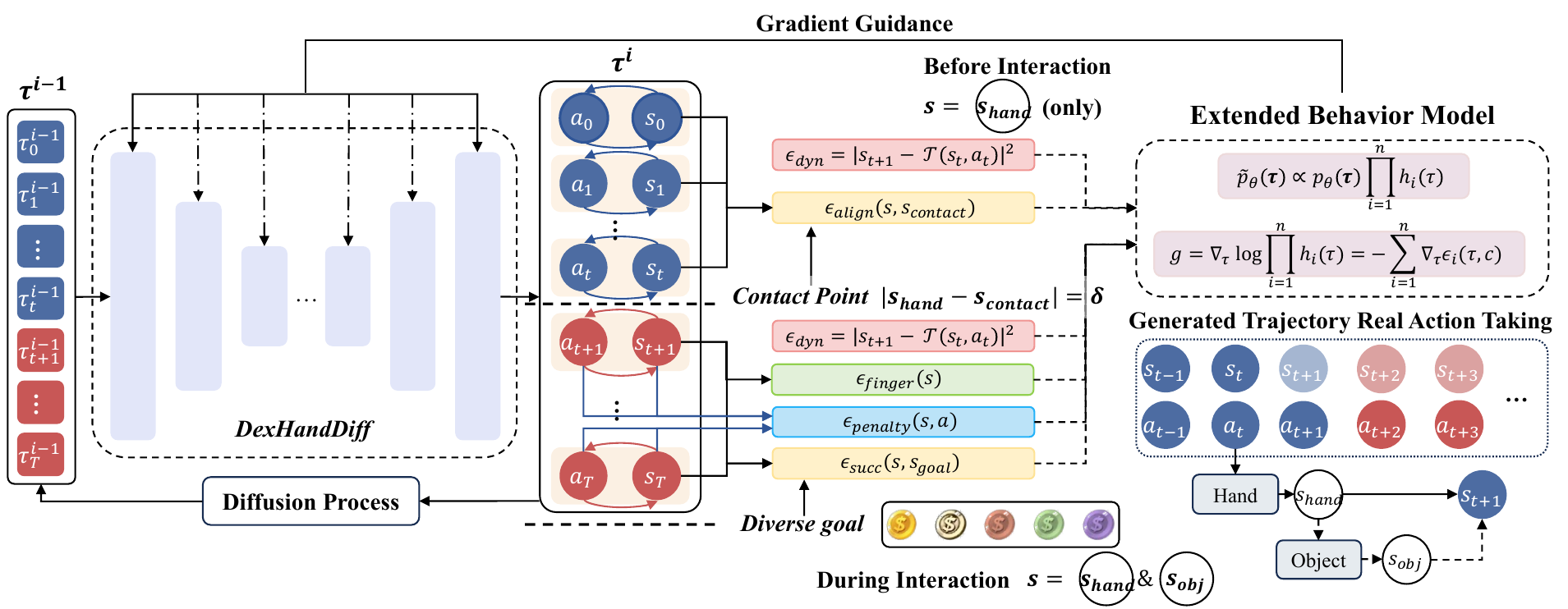}
   \vspace{-10pt}
      \caption{\textbf{Framework of \alias.} \alias employs joint state-action diffusion with interaction-aware guidance. Before interaction (top middle), guidance aligns the hand to the object contact point. Upon contact (bottom middle), additional guidance steers both hand and object states towards the goal (``\&'' means state concatenation at input level), enforcing physical constraints and avoiding ghost states. A learned dynamics model further ensures consistency between states and actions. Our~\alias utilizes extended behavior model to aggregate multiple condition terms to guide the diffusion process.
      % This extended behavior model-based framework ensures adaptive, realistic control for manipulation.
      }
   \label{fig:framework}
   \vspace{-15pt}
\end{figure*}

\myparagraph{Joint State-Action Diffusion Model.}
Our approach builds upon classifier-guided diffusion policies. But we jointly diffuse over the concatenated state-action space $\btau = [(\ba_0, \bs_0), (\ba_1, \bs_1), ..., (\ba_T, \bs_T)]$, where state $\bs$ includes both hand (24 joint angles and 3 position offsets) and task-specific object states (\ie door hinge angle, pen pose \etc), and action $\ba$ represents changes in controllable states (only hand joints and positions).

This design choice directly addresses the above mentioned limitations: (1) By including states in the diffusion process, we enable explicit state conditioning and goal specification, overcoming the limitations of action-only approaches; (2) By classifier-guided diffusion, we allow flexible goal adaptation without exhaustive training data; (3) By jointly modeling states and actions, we maintain their physical coupling and prevent ghost states through carefully designed guidance. With denoised states guiding the generated actions, we effectively balance the state conditioning and action precision.

\myparagraph{Extended Behavior Model and Energy Function.}
% Building upon the basic classifier-guided model (Sec.~\ref{sec:c-guided formulation}), we extend the formulation to accommodate multiple guidance (or constraints) simultaneously for complex interaction tasks. 
According to Eq.~\ref{eq:perturbed}, the standard conditional diffusion follows:
\vspace{-8pt}
\begin{equation}
\tilde{p}_\theta(\btau{}) \propto p_{\theta}(\btau{}) p(\mathcal{O}_{1:T}=1 \mid \btau{}) \propto p_{\theta}(\btau{})h(\btau{}),
% \vspace{-2pt}
\end{equation}
where we generalize $p(\mathcal{O}_{1:T}=1 \mid \btau{})$ as a behavior model $h(\btau{})$. Then we further generalize this formulation through a product of experts framework~\cite{product_of_expert}, where each expert represents a specific behavior model:
\vspace{-8pt}
{\small
\begin{equation}
\tilde{p}_\theta(\btau{}) \propto p_\theta(\btau{}) \prod_{i=1}^n h_i(\btau{}).
\label{eq:product_experts}
\vspace{-6pt}
\end{equation}}

From the energy function perspective, each behavior model encoding task-specific objectives or constraints is:
\vspace{-7pt}
{\small
\begin{equation}
h_i(\btau{}, c) = \frac{1}{\int e^{-\varepsilon_i(\btau{},c)}d\btau{}} e^{-\varepsilon_i(\btau{},c)},
\label{eq:energy_guide}
\vspace{-6pt}
\end{equation}}
where $\varepsilon_i(\btau{},c)$ represents the energy function for the $i$-th guidance objective, with $c$ denoting task-specific conditions. This formulation allows combining multiple objectives (\eg, reaching the target state while maintaining physical consistency) via their respective guidance functions.

Under appropriate smoothness conditions, the guidance gradient $g$ in the reverse diffusion process (Eq.~\ref{eq:guided}) can be decomposed as the sum of individual guidance gradients:
\vspace{-6pt}
{\small
\begin{equation*}
% \begin{aligned}
g = \nabla_{\btau{}} \log \prod_{i=1}^n h_i(\btau{})= \sum_{i=1}^n \nabla_{\btau{}} \log h_i(\btau{}) = -\sum_{i=1}^n \nabla_{\btau{}} \varepsilon_i(\btau{},c).
\label{eq:combined_gradient}
% \end{aligned}
\vspace{-5pt}
\end{equation*}}

This enables integration of multiple guidance signals, each addressing different aspects of the interaction task, while maintaining a coherent optimization objective.

\myparagraph{Dynamics-aware Generation.}
A key challenge in joint state-action diffusion is maintaining consistency between generated states and actions~\cite{diffuser}. Our method addresses this through a learned dynamics model trained on demonstration data, constraining state-action generation via additional loss
in diffusion training and serving as a guide in inference.
By penalizing state-action pairs that violate observed dynamics, this guidance ensures our model maintains both state conditioning benefits and action feasibility. 
{\small
\vspace{-8pt}
\begin{equation}
\varepsilon_{\text{dyn}}(\btau) = |\bs_{t+1} - \mathcal{T}(\bs_t, \ba_t)|^2,
\label{eq:dynamics_energy}
\vspace{-2pt}
\end{equation}
}
where $\mathcal{T}(\bs,\ba)$ is a separately trained dynamics model to ensure physically plausible motion patterns.

% \myparagraph{Dual-phase Interaction Framework.}
\myparagraph{Manipulation after Contact Task Guidance.}
For manipulation after contact tasks such as door opening and tool using, \alias employs a dual-phase interaction approach that acknowledges the fundamentally different nature of interaction before and after contact establishment. The framework automatically determines the phase transition based on the distance between the palm position and the designated contact point on the object, applying a smooth transition mask to blend between phases.

In the pre-grasp phase, our method focuses on guiding the manipulator to stably align with the contact point while preventing premature object movement. We engineer two primary guidance components:
1) Alignment guidance $\epsilon_\text{align}$ that directs the end-effector towards precise contact point while maintaining natural approaching trajectory;
2) Dynamics consistency guidance $\epsilon_\text{dyn}$.
% that leverages a separately trained transition model $\tilde{\mathcal{T}}(\bs,\ba)$ to ensure physically plausible motion patterns.

Upon establishing contact (determined by palm-object proximity), the post-grasp phase activates additional guidance mechanisms:
1) Goal-directed guidance $\epsilon_\text{succ}$ that steers the coupled hand-object system towards target configurations;
2) Physical constraint guidance $\epsilon_\text{penalty}$ that prevents unrealistic state changes (\eg, limiting per-step changes in both door hinge and latch angles);
3) Continued dynamics guidance $\epsilon_\text{dyn}$ to maintain motion feasibility.

Therefore, the guidance energy function follows,
\vspace{-6pt}
{\small
\begin{equation}
\epsilon = \begin{cases}
\epsilon_{\text{pre}} = \epsilon_{\text{align}} + \epsilon_{\text{dyn}} & \text{if } |\boldsymbol{s}_{\text{hand}} - \boldsymbol{s}_{\text{contact}}| > \delta_1 \\
\epsilon_{\text{post}} = \epsilon_{\text{succ}} + \epsilon_{\text{dyn}} + \epsilon_{\text{penalty}} & \text{otherwise}
\vspace{-10pt}
\end{cases}
\vspace{-1pt}
\end{equation}}
where $\boldsymbol{s}_{\text{hand}}$ and $\boldsymbol{s}_{\text{contact}}$ represents the states of dexterous hand and object contact point (\eg door latch, hammer handle \etc) respectively, and $\delta_1$ is a small threshold.
The separated design of grasp proposal guidance ($\epsilon_{\text{align}}$) and task achieving guidance ($\epsilon_{\text{succ}}$) mirrors successful policies in prior work~\cite{wu2024unidexfpm,wan2023unidexgrasp++}, effective for dexterous manipulation.
Besides, the $\epsilon_{\text{penalty}}$ ensures continuous object state transitions, corresponding to 
\vspace{-6pt}
\begin{equation}
    h_{penalty}\triangleq1-H(|s_{obj}^{t+1}-s_{obj}^{t}|-\delta_2),
    \vspace{-4pt}
\end{equation}
where $\delta_2$ is another small threshold and $H(\cdot)$ is the Heaviside step function~\cite{weisstein2002heaviside}. Then $\epsilon_{\text{penalty}}$ can be obtained by applying Eq.~\ref{eq:energy_guide}, becoming a Dirac delta function that directly sets value when satisfying the constraints.

\myparagraph{In-hand Manipulation Task Guidance.}
For tasks primarily involving in-hand manipulation (e.g., pen spinning, object reorientation), where objects are typically already in hand or quickly transition to in-hand states, we employ a simplified single-phase guidance structure:
1) Goal state guidance $\epsilon_{\text{succ}}$ for achieving target object configurations;
2) Active finger motion guidance to ensure realistic object manipulation;
3) Dynamics consistency guidance $\epsilon_{\text{dyn}}$ to maintain physical plausibility;
4) Physical constraint guidance $\epsilon_{\text{penalty}}$ that prevents unrealistic state changes.
{\small
\vspace{-5pt}
\begin{equation}
\epsilon = \epsilon_{\text{goal}} + \epsilon_{\text{finger}} + \epsilon_{\text{dyn}} + \epsilon_{\text{penalty}}.
\vspace{-4pt}
\end{equation}
}

Specially, we define the behavior model that encourages active finger involvement as,
{\small
\vspace{-4pt}
\begin{equation}
    h_{\text{finger}}(\btau, t) = H(|\bs^{t+1}_{\text{finger-joints}} - \bs^{t}_{\text{finger-joints}}| - \delta_3),
    \vspace{-4pt}
\end{equation}
}
where $\bs^{t}_{\text{finger-joints}}$ is the state vector of all finger joints at planning step $t$. $\delta_3$ is the third small threshold. $H(\cdot)$ is also the Heaviside step function.
This specialized handling prevents unrealistic ``ghost states", as discussed in Sec.~\ref{sec:challenge_of_existing}.

\subsection{LLM-Based Guidance Generation}
\label{sec:llm_gen}

The design of task-specific guidance functions for diffusion policies traditionally requires significant manual effort, particularly for diverse dexterous manipulation tasks. To address this challenge, we leverage a \textbf{two-stage} Large Language Model (LLM) process for automated guidance generation, adopting text-to-reward paradigm~\cite{eureka,text2reward}. 
% Our classifier-guided diffusion framework provides a natural interface for such LLM-generated guidance functions through its explicit conditioning mechanism.

\myparagraph{Overall Pipeline.} First, we feed the LLM with a 6-part template (including function purpose, guidance structure, environment description, function prototype, task instruction and few-shot hints) and public documents on simulation environments~\cite{adroit,shadow-hand-env} to generate task-specific prompts. Then, the generated task prompts are queried to another LLM to write guidance function code. Only few-shot hints require specific refinement, reducing human trial-and-error times from about 20 (for hand-craft energy function design) to around 5 while maintaining \alias performance.

\myparagraph{Environment Description.}
Our approach employs a comprehensive \textit{Pythonic} environment abstraction that captures the complete interaction system. It encapsulates detailed robot joint configurations, and object-environment specifications from public documents, enabling LLM to generate precise guidance functions that account for the full complexity of dexterous manipulation tasks. 
% The structured representation allows for clear specification of both the robot's capabilities and the task settings.

\myparagraph{Other Details.}
As previous works~\cite{text2reward}, once the guidance function code is generated, we execute the code in interpreter. This step may give us valuable feedback, \eg, syntax errors and runtime errors. We utilize the feedback from code execution as a tool for ongoing refinement within the LLM. Besides, our approach uses few-shot hints instead of examples to allow the model to access relevant functions and best practices without direct examples. Each guidance component is normalized over the trajectory horizon to ensure balanced contributions across objectives while preserving their temporal structure.
%
% To ensure effective combination of multiple guidance components, our framework incorporates horizon-level magnitude normalization. Each guidance component is normalized across the trajectory horizon before combination, enabling balanced contribution from different objectives while maintaining their temporal characteristics. This normalization is important when combining various task-specific guidance functions generated from natural language descriptions. 
Detailed examples of prompts and generated guidance functions are shown in Appx.~\ref{appendix:prompts}.

\begin{table*}[tbh]
  \centering
  \small
  \resizebox{0.99\linewidth}{!}{
  \begin{tabular}{c|c|ccc|c|cc|c}
    \toprule
    \textbf{Method} & \textbf{Condition} & \textbf{Open 30$^\circ$}& \textbf{Open 50$^\circ$}& \textbf{Open 70$^\circ$} & \textbf{Open 90$^\circ$} & \textbf{Open 110$^\circ$} & \textbf{Close Door} & \textbf{Average} \\
    \midrule
    \textbf{Diffuser}~\cite{diffuser} & Goal Inpainting & $16.7$ \scriptsize{\raisebox{1pt}{$\pm 4.7$}} & $16.7$ \scriptsize{\raisebox{1pt}{$\pm 12.5$}} & $6.7$ \scriptsize{\raisebox{1pt}{$\pm 4.7$}} & $56.7$ \scriptsize{\raisebox{1pt}{$\pm 9.4$}} & $10.0$ \scriptsize{\raisebox{1pt}{$\pm 8.2$}} & $0$ & $17.8$\\

    \textbf{Diffuser}~\cite{diffuser} & Guided Sampling & $10.0$ \scriptsize{\raisebox{1pt}{$\pm 8.2$}} & $26.7$ \scriptsize{\raisebox{1pt}{$\pm 17.0$}} & $10.0$ \scriptsize{\raisebox{1pt}{$\pm 4.7$}} & $63.3$ \scriptsize{\raisebox{1pt}{$\pm 18.7$}} & $6.7$ \scriptsize{\raisebox{1pt}{$\pm 9.4$}} & \textbf{60.0} \scriptsize{\raisebox{1pt}{$\pm 8.2$}} & $29.5$ \\
    \midrule
    
    \textbf{Decision Diffuser}~\cite{decisiondiffuser} & Embedding & $0$ & $3.3$ \scriptsize{\raisebox{1pt}{$\pm 4.7$}} & $16.7$ \scriptsize{\raisebox{1pt}{$\pm 4.7$}} & \textbf{100} \scriptsize{\raisebox{1pt}{$\pm 0$}} \hspace{3pt} & \textbf{30.0} \scriptsize{\raisebox{1pt}{$\pm 8.2$}} & $0$ & $25.0$\\
    \midrule
    
    \textbf{Diffusion Policy}~\cite{diffusion_policy} & Embedding & $16.7$ \scriptsize{\raisebox{1pt}{$\pm 4.7$}} & $3.3$ \scriptsize{\raisebox{1pt}{$\pm 4.7$}} & $13.3$ \scriptsize{\raisebox{1pt}{$\pm 12.5$}}  & \textbf{100} \scriptsize{\raisebox{1pt}{$\pm 0$}} \hspace{3pt} & $3.3$ \scriptsize{\raisebox{1pt}{$\pm 4.7$}} & $0$ & $22.8$ \\
    \midrule

    % \textbf{Joint-SA Diffuser} & Guided Sampling & $0$ & $23.3$ \scriptsize{\raisebox{1pt}{$\pm 20.6$}} & $20.0$ \scriptsize{\raisebox{1pt}{$\pm 8.2$}}  & $80.0$ \scriptsize{\raisebox{1pt}{$\pm 8.2$}} & $6.7$ \scriptsize{\raisebox{1pt}{$\pm 4.7$}} & $23.3$ \scriptsize{\raisebox{1pt}{$\pm 12.5$}} & $25.6$ \\
    % \midrule
    
    \textbf{\alias-disc.} & Goal Inpainting & $46.7$ \scriptsize{\raisebox{1pt}{$\pm 4.7$}} & $13.3$ \scriptsize{\raisebox{1pt}{$\pm 9.4$}} & \textbf{53.3} \scriptsize{\raisebox{1pt}{$\pm 4.7$}} & $20.0$ \scriptsize{\raisebox{1pt}{$\pm 8.2$}} & $6.7$ \scriptsize{\raisebox{1pt}{$\pm 4.7$}} & $0$ & $23.3$\\
    % \midrule
    
    \bestcell{\textbf{\alias (Ours)}} & \bestcell{Guided Sampling} & \bestcell{\textbf{70.0} \scriptsize{\raisebox{1pt}{$\pm 8.2$}}} & \bestcell{\textbf{56.7} \scriptsize{\raisebox{1pt}{$\pm 4.7$}}} & \bestcell{\textbf{53.3} \scriptsize{\raisebox{1pt}{$\pm 8.2$}}} & 
    \bestcell{$90.0$ \scriptsize{\raisebox{1pt}{$\pm 8.2$}}} & 
    % \bestcell{$83.3$ \scriptsize{\raisebox{1pt}{$\pm 12.5$}}} & 
    \bestcell{\textbf{26.7} \scriptsize{\raisebox{1pt}{$\pm 14.1$}}} & \bestcell{\textbf{58.3} \scriptsize{\raisebox{1pt}{$\pm 13.4$}}} & \bestcell{\textbf{59.2}}\\
    \bottomrule
  \end{tabular}}
  \vspace{-6pt}
  \caption{\textbf{Success rates (in \%) of different diffusion-based approaches in Adroit Hand~\cite{adroit} environment.} All models were trained on the Open 90° task only, and we test their adaptability to other task goals in Adroit Door environment. All results and standard deviation are calculated over 3 tries for 10 random seeds. Best methods and those within 5\% of the best are highlighted in \textbf{bold}.}
  \vspace{-4pt}
  \label{tab:quantative}
\end{table*}

\section{Experiments}
\label{sec:experiment}

% \subsection{Environment and Datasets}
% \vspace{-3pt}
% \myparagraph{Environment.}
We evaluate our \alias on five challenging dexterous manipulation tasks with four from Adroit Hand~\cite{adroit} and one from Shadow Hand environment~\cite{shadow-hand-env}.
Both environments feature a 24-joint Shadow Hand simulator with up to 30 degrees of freedom, designed to closely match the hardware setting~\cite{shadowrobot}. Detailed explanations of the five tasks are provided in Appendix~\ref{append:environ}. 
% For Adroit tasks (door opening, hammer striking, and pen reorientation), we use the expert demonstrations from D4RL~\cite{d4rl} collected by teleoperation for training. For the block rotate-Z task, we collect 5000 expert trajectories using TQC+HER~\cite{TQC,HER}.
%
We use the expert demonstrations collected by teleoperation from D4RL~\cite{d4rl} for Adroit tasks (Door, Hammer, Pen and Relocate). However, Shadow Hand environment does not provide demonstration data, so we employ TQC+HER~\cite{TQC,HER} to collect \textbf{5000} expert trajectories for the Block Rotate-Z task.

% The door task represents multi-stage manipulation where the hand must reach and rotate a door handle, then pull or push the door to a target angle. The hammer task tests tool use capabilities, requiring the hand to grasp the hammer and strike a nail, while the pen and the block task evaluates in-hand dexterity, targeting continuous object reorientation.

\begin{table*}[tbh]
  \centering
  \small
  \resizebox{0.9\linewidth}{!}{
  \setlength{\tabcolsep}{10pt}
  \begin{tabular}{cc|ccc}
    \toprule
    \textbf{Environment} & \textbf{Task} & \makecell{\textbf{Diffuser~\cite{diffuser}} (Inpaint)} & \textbf{\makecell{Conditional DP~\cite{decisiondiffuser,diffusion_policy}}} & \bestcell{\textbf{\makecell{\alias (Ours)}}} \\
    \midrule
     Door & Open 90$^{\circ}$ & $56.7$ \scriptsize{\raisebox{1pt}{$\pm 9.4$}} & \textbf{100} \scriptsize{\raisebox{1pt}{$\pm 0$}} \hspace{2.5pt} & \bestcell{$90.0$ \scriptsize{\raisebox{1pt}{$\pm 8.2$}}}  \\
     Door & Open 30$^\circ$ & $16.7$ \scriptsize{\raisebox{1pt}{$\pm 4.7$}} & $16.7$ \scriptsize{\raisebox{1pt}{$\pm 4.7$}} & \bestcell{\textbf{70.0} \scriptsize{\raisebox{1pt}{$\pm 8.2$}}}  \\
     
     \midrule
     Pen & Full Re-orientation & $10.0$ \scriptsize{\raisebox{1pt}{$\pm 0$}}\hspace{5pt} & $80.0$ \scriptsize{\raisebox{1pt}{$\pm 8.2$}} & \bestcell{\textbf{93.3} \scriptsize{\raisebox{1pt}{$\pm 4.7$}}} \\
     Pen & Half-side Re-orientation & \hspace{3pt} $3.3$ \scriptsize{\raisebox{1pt}{$\pm 4.7$}} & $23.3$ \scriptsize{\raisebox{1pt}{$\pm 9.4$}} & \bestcell{$\textbf{40.0}$ \scriptsize{\raisebox{1pt}{$\pm 8.2$}}} \\
     \midrule
     Hammer & Nail Full Drive & $53.3$ \scriptsize{\raisebox{1pt}{$\pm 9.4$}} & $76.7$ \scriptsize{\raisebox{1pt}{$\pm 9.4$}} & \bestcell{\textbf{90.0} \scriptsize{\raisebox{1pt}{$\pm 8.2$}}} \\
     Hammer & Nail Half Drive & \hspace{1pt} $23.3$ \scriptsize{\raisebox{1pt}{$\pm 12.5$}} & $33.3$ \scriptsize{\raisebox{1pt}{$\pm 4.7$}} & \bestcell{\hspace{1.5pt} \textbf{46.7} \scriptsize{\raisebox{1pt}{$\pm 12.5$}}}\\

     \midrule
    Relocate & Full Relocation & $56.7$ \scriptsize{\raisebox{1pt}{$\pm 4.7$}}\hspace{5pt} & $\textbf{96.7}$ \scriptsize{\raisebox{1pt}{$\pm 4.7$}} & \bestcell{\textbf{96.7} \scriptsize{\raisebox{1pt}{$\pm 4.7$}}} \\
     Relocate & Half-side Relocation & $53.3$ \scriptsize{\raisebox{1pt}{$\pm 4.7$}} & \hspace{1.5pt} $86.7$ \scriptsize{\raisebox{1pt}{$\pm 12.5$}} & \bestcell{$\textbf{93.3}$ \scriptsize{\raisebox{1pt}{$\pm 4.7$}}} \\
     
     \midrule
     Manipulate Block & Rotate-Z & \hspace{1pt} $36.7$ \scriptsize{\raisebox{1pt}{$\pm 12.5$}} & $40.0$ \scriptsize{\raisebox{1pt}{$\pm 8.2$}} & \bestcell{\textbf{50.0} \scriptsize{\raisebox{1pt}{$\pm 8.2$}}} \\
     Manipulate Block & Half-side Rotate-Z & $30.0$ \scriptsize{\raisebox{1pt}{$\pm 0$}} \hspace{5pt} & $26.7$ \scriptsize{\raisebox{1pt}{$\pm 4.7$}} & \bestcell{\textbf{36.7} \scriptsize{\raisebox{1pt}{$\pm 4.7$}}}\\
     \midrule
     \multicolumn{2}{c|}{\textbf{Average}} & 34.0 \hspace{.58cm} & 58.0 \hspace{.58cm} & \textbf{70.7\hspace{.68cm}} \\
    \bottomrule
  \end{tabular}}
  \vspace{-7pt}
  \caption{\textbf{Overall performance of dexterous manipulation with goal adaptability on multiple environments and tasks.} We compare our method with one classifier-guided baseline and one classifier-free baseline. The results are calculated over 3 tries for 10 random seeds.}
  \vspace{-12pt}
  \label{tab:overall}
\end{table*}

% \myparagraph{Settings.} 

\subsection{Performance Comparisons on Goal Adaptability in Interaction-Aware Tasks}

We evaluate \alias in the Door environment to test its goal adaptability across various target angles. Specifically, we require the planners to open the door to 30, 50, 70, 90 and 110 degrees, as well as close the door (reversal task). Note that the training data only includes 90-degree door-opening demonstrations.
For some of these tasks, we adjust the environment settings, such as expanding the door's range of motion, to satisfy the evaluation requirements.
% and create distinct challenges of adaptability. 

We compare \alias with five baselines: two classifier-guided methods (Diffuser~\cite{diffuser} with Goal Inpainting that sets discrete goal states, and Diffuser with Guided Sampling that leverages continuous gradients for fine control), two classifier-free methods (Decision Diffuser~\cite{decisiondiffuser} and Diffusion Policy~\cite{diffusion_policy} that apply diffusion on states and actions respectively), and a variant of \alias (denoted \alias-disc.) that uses goal inpainting. To enhance classifier-free methods' learning of goal condition, we use~\textit{the difference between the current door angle and target angle} as the condition, rather than a fixed 90$^\circ$ target.

% across goal-adaptive tasks

The results are shown in Tab.~\ref{tab:quantative}. Classifier-free methods perform well on the 90$^\circ$ task, but their success declines sharply on new target angles, indicating limited adaptability to out-of-distribution targets. Classifier-guided methods demonstrate moderate but consistent performance across goal-adaptive tasks yet their overall success rates remain suboptimal due to imprecise state-action relation modeling in the policy.
%
% As shown in Tab.~\ref{tab:quantative}, classifier-free methods perform well on the 90$^\circ$ task, consistent with the training data, but their success declines sharply on new target angles, indicating limited adaptability to out-of-distribution targets. Classifier-guided methods achieve only modest success on goal-adaptive tasks, with state-only diffusion leading to less feasible actions due to errors from inverse dynamics model.
%
Our \alias
% with its interaction-aware, joint state-action diffusion and guided sampling, 
achieves consistently high success rates across nearly all tasks. 
% While it achieves $90.0\%$ success on the training task (90$^\circ$) compared to $100\%$ of classifier-free methods, 
The slightly lower performance ($90.0\%$) on the training task (90$^\circ$) compared to classifier-free methods stems from our additional guidance for adaptation. When ablating this,
\alias achieves $96.7\pm4.7\%$ success rate on Open 90$^\circ$, which is a reasonable trade-off for better generalization.
Averaging a 59.2\% success rate, over twice that of the next best method (29.5\%), \alias demonstrates robust adaptability across both in-domain and goal-adaptive scenarios.

Besides, we also observe a trend that goals closer to the original training data don't have higher success rates than others. We suppose it's because when target angle is close to training, learned dynamics often override guidance. We
observed 8 out of 14 failures in 30 tries in 70$^\circ$ task opened to 90$^\circ$ instead, supporting our hypothesis. This learned bias is harder to correct than for more distant angles.
% Despite enhanced condition embeddings, classifier-free methods still under-perform relative to \alias, addressing limitations of previous diffusion-based models on dexterous manipulation tasks.

\begin{figure*}[tb]
  \centering
   \includegraphics[width=0.99\linewidth]{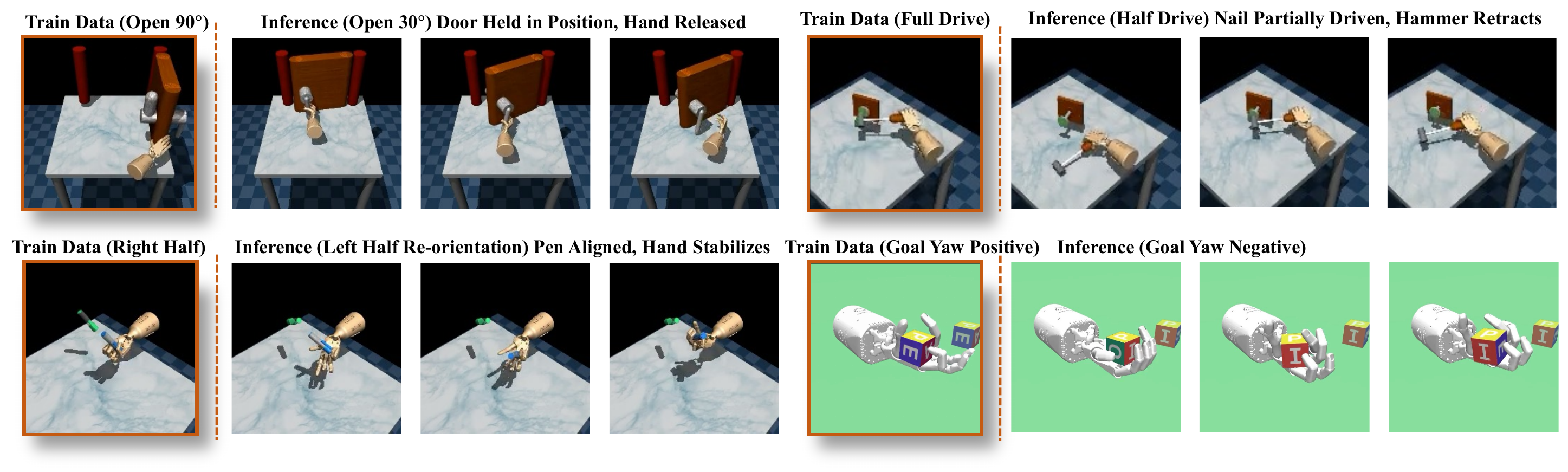}
   \vspace{-12pt}
   \caption{\textbf{Visualization results of goal-adaptive tasks by \alias.} For each task, training data sample (with \textcolor{orange}{orange stroke}) is followed by inference on novel goals beyond the training data. In the Door task, \alias guides the door to new target angle (30$^\circ$) and holds the door in position when the hand releases, \textit{which cannot be attained by simply truncating actions from 90$^\circ$ training data}. \alias avoids ghost states and achieves better goal adaptability.
   % Similarly, \alias re-orients the pen or the block, stabilizes the hand, and drives the nail partially before retracting the hammer, avoiding ghost states and achieving goal adaptability.
   }
   \label{fig:visualization}
   \vspace{-14pt}
\end{figure*}

\subsection{Evaluation on Various Dexterous Tasks}

To evaluate the cross-task adaptability and goal-oriented performance of \alias, we test it across multiple dexterous manipulation tasks in Door, Pen, Hammer, Relocate and Block environments, as summarized in Tab.~\ref{tab:overall}. In addition to the Door task, the Pen task involve aligning a pen to the specified orientation, with a particularly challenging goal-adaptability variant, Half-side Re-orientation, where training data includes only right-hemisphere orientations while test goals require left-hemisphere rotations. Similarly, the Block Rotate-Z and Object Relocation tasks have block's half-side variant trained on positive goal yaw angles but tested on negative ones and object's target right-half table training but left-half testing. The Nail Half Drive requires the hand to drive a nail and stop halfway before retracting, testing control precision for partial goals.

We compare \alias with two baselines: Diffuser~\cite{diffuser} (Inpainting), using classifier-guided goal inpainting as in the previous section, and Conditional DP~\cite{decisiondiffuser,diffusion_policy}, a classifier-free approach with state diffusion for Door, Hammer and Relocate tasks while action diffusion for Pen and Block tasks, as modeling dynamics for these tasks are particularly challenging, making direct action generation more effective than state-based diffusion. As shown in Tab.~\ref{tab:overall}, \alias consistently achieves superior results across both in-domain and goal-adaptive tasks. 
% For instance, \alias achieves 93.3\% success rate on pen full re-orientation (in-domain) compared to Conditional DP's 80\% and Diffuser's 10\%, and 46.7\% on nail half drive (goal-adaptive) vs.~23.3\% for Diffuser and 33.3\% for Conditional DP.
%
Although conditional DP demonstrates 23.3\% on the challenging pen half-side re-orientation, leveraging the inherent multi-modality and anisotropy of diffusion models, \alias still performs better (40.0\%). These results underscore \alias's robustness and adaptability across a range of tasks, demonstrating generalization on familiar goals and novel configuration challenges.

\subsection{Validation for Preventing Ghost States}
We measured \textit{L2 distance} between predicted and simulated hand-object states (normalized per dimension for fair comparison) in Tab.~\ref{tab:ghost_state}. \alias nearly halves baseline's gap across tasks, illustrating its ghost-state reduction effect. 
% Further analysis can be found in \textcolor{MyDarkGreen}{\textbf{Q3}}.

\begin{table}[tb]
\centering
\small
\tabcolsep3.5pt
\resizebox{1.\linewidth}{!}{
\begin{tabular}{l |c c c c c}
\toprule
\textbf{Adapt Tasks} & \textbf{Door 30$^\circ$} & \textbf{Door 70$^\circ$} & \textbf{Pen Half} & \textbf{Hammer Half} & \textbf{Relocate Half}\\
\midrule
% \textbf{Diffuser (guideFunc)} & 4.44 & 4.75 & - & - & - \\
\textbf{Diffuser~\cite{diffuser}} & 4.19 & 4.03 & 5.23 & 4.01 & 5.48\\
\bestcell{\textbf{\alias}} & \bestcell{\textbf{2.92}} & \bestcell{\textbf{2.38}} & \bestcell{\textbf{2.76}} & \bestcell{\textbf{2.41}} & \bestcell{\textbf{3.22}}\\
\bottomrule
\end{tabular}}
\vspace{-7.5pt}
  \caption{\textbf{Quantitative results for preventing ghost states over 3 tries.} (Conditional DP is not included due to its action-only.)}
    \label{tab:ghost_state}
    \vspace{-15pt}
\end{table}

\subsection{Ablation on LLM-based Guidance Generation}
Table~\ref{tab:LLM_gen} presents results for different guidance methods on goal adaptability tasks. All three methods are based on the same joint state-action diffusion model. The Human Craft approach reflects our above results with manually designed guidance. LLM Gen generate guidance functions with Claude Sonnet 3.5~\cite{claude_3_5}. And Na\"\i ve Guide directly guides the object to the goal, corresponding to ghost-state existing baseline. 
Results indicate that both Human Craft and LLM Gen significantly outperform Na\"\i ve Guide across tasks, with Human Craft achieving the highest success rates.

\begin{table}[t]
  \centering
  \small
  \resizebox{\linewidth}{!}{
  \begin{tabular}{cc|cc}
    \toprule
    \textbf{Task} & \textbf{Na\"\i ve Guide} & \textbf{Human Craft} & \textbf{LLM Gen}\\
    \midrule
    Door Open 30$^{\circ}$ & $0$ & $70.0$ \scriptsize{\raisebox{1pt}{$\pm 8.2$}} & $40.0$ \scriptsize{\raisebox{1pt}{$\pm 8.2$}} \\
    Pen Half-side Re-orien & $20.0$ \scriptsize{\raisebox{1pt}{$\pm 8.2$}} & $40.0$ \scriptsize{\raisebox{1pt}{$\pm 8.2$}} & $26.7$ \scriptsize{\raisebox{1pt}{$\pm 4.7$}} \\
    Hammer Half Nail & $20.0$ \scriptsize{\raisebox{1pt}{$\pm 8.2$}} & \hspace{1pt} $46.7$ \scriptsize{\raisebox{1pt}{$\pm 12.5$}} & $43.3$ \scriptsize{\raisebox{1pt}{$\pm 9.4$}}\\
    \bottomrule
  \end{tabular}}
  \vspace{-8pt}
  \caption{\textbf{Ablation study on LLM-based guidance generation.}}
  \vspace{-14pt}
  \label{tab:LLM_gen}
\end{table}

\subsection{Ablation Study of \alias Framework}

We analyze the contribution of each component in \alias through ablation studies (Tab.~\ref{tab:ablation}), across multiple door-opening tasks (open 30$^\circ$, 50$^\circ$, 70$^\circ$, and 90$^\circ$), using the same training checkpoint for fair comparison. The baseline Diffuser\cite{diffuser} uses a basic goal-guidance strategy, while Dyn-guide enhances it with dynamics guidance for better state-action consistency. Joint S\&A adopts joint state-action denoising like \alias but retains naive goal guidance. \alias incorporates all components and achieves the highest success rate of 67.5\%, significantly outperforming the other configurations and demonstrating the effectiveness of our full design.

\subsection{Visualizations}
We visualize the behavior of \alias across various goal-adaptive dexterous tasks in Fig.~\ref{fig:visualization}.
% illustrates the interaction-aware behavior of \alias across various goal-adaptive dexterous tasks.
% , contrasting prior methods that exhibit ghost states.
%
% Each task visualization includes a sample from training and corresponding goal-adaptive execution by \alias. 
\alias ensures realistic contact by aligning hands with contact points first using joint dynamics modeling, eliminating ghost states.
Notably, for example, \alias guides the door to new target angle and holds the door steady when the hand releases , which cannot be achieved by policies trained with slicing 90$^\circ$ data.
%
% For instance, in the Door tasks, \alias guides the hand to grasp the handle before adjusting the door to target angles, holding the door steady as the hand releases, which is unachievable by simply slicing 90$^\circ$ training data.
%
% Similarly, in the Pen Re-orientation, Block Rotate-Z and Hammer Nail Drive tasks, \alias effectively manages large re-orientations and phased control: the hand rotates the pen over a wide arc, and the hammer strikes the nail partially before retracting, ensuring smooth, contact-driven transitions throughout. 
These results underscore \alias's ability to maintain physically realistic interactions while adapting to novel goals.

\begin{table}[t]
  \centering
  \small
  \resizebox{1.\linewidth}{!}{
  \begin{tabular}{cccccc}
    \toprule
    \textbf{Method} & \textbf{\makecell{Goal\\Guidance}} & \textbf{\makecell{Dynamics\\Guide}} & \textbf{\makecell{Joint State\\Action}} & \textbf{\makecell{Interact\\Mechanism}} & \textbf{\makecell{Overall\\SR}}\\
    \midrule
     \textbf{No-guide} & $\times$ & $\times$ & $\times$ & $\times$ & $24.1$\\
     \textbf{Diffuser}~\cite{diffuser} & \checkmark & $\times$ & $\times$ & $\times$ & $27.5$\\
     \textbf{Dyn-guide} & \checkmark & \checkmark & $\times$ & $\times$ & $27.5$\\
     \textbf{Joint S\&A} & \checkmark & $\times$ & \checkmark & $\times$ & $30.8$\\
     \textbf{Dyn+Joint} & \checkmark & \checkmark & \checkmark & $\times$ & $31.7$\\
     \bestcell{\textbf{\alias}} & \bestcell{\checkmark} & \bestcell{\checkmark} & \bestcell{\checkmark} & \bestcell{\checkmark} & \bestcell{\textbf{67.5}}\\
    \bottomrule
  \end{tabular}}
  \vspace{-8pt}
  \caption{\textbf{Ablation study on \alias framework.} We report the average success rates (overall SR) on Adroit Door environment over open 30$^\circ$, 50$^\circ$, 70$^\circ$ and 90$^\circ$ tasks.}
  \vspace{-12pt}
  \label{tab:ablation}
\end{table}

\subsection{Efficiency}
We test the control frequency of \alias on an RTX 3090 with receding horizon set as 8 for all tasks except Door (32 instead). The control frequency are reported below.

\vspace{-6pt}
\begin{table}[h]
\centering
\small
\tabcolsep3.5pt
    \resizebox{0.75\linewidth}{!}{
\begin{tabular}{l |c c c c c}
\toprule
\textbf{Task} & \textbf{Door} & \textbf{Pen} & \textbf{Hammer} & \textbf{Relocate} & \textbf{Block}\\
\midrule
\textbf{Freq.} & 5.04 Hz & 5.88 Hz &  5.86 Hz & 5.78 Hz & 6.92 Hz\\
\bottomrule
\end{tabular}}
\vspace{-8pt}
  \caption{\textbf{Control command frequency over 10 tries.}}
    \label{tab:frequency}
    \vspace{-11pt}
\end{table}

Besides, our lightweight model (3.96M~params, 3.27~GFLOPS) can be further accelerated via DPM Solver++~\cite{dpmsolver++} (4x~speedup) and command interpolation (reaching 36 Hz), sufficient for real robot control.

\section{Conclusion}
\label{sec:conclusion}

This work presents \alias, an interaction-aware diffusion planner for adaptive dexterous manipulation. 
% that can generalize to diverse task goals even in contact-rich scenarios.
%
By modeling joint state-action dynamics and incorporating a dual-phase diffusion mechanism, it addresses action-state consistency issues, including the ``ghost state'' and generalization problems observed in previous diffusion methods.
\alias's design enables it to handle intricate multi-contact interactions through a pre-contact alignment and a post-contact control.
% ensuring dynamics-based and physics-realistic interactions
%
% These results with various settings on interaction tasks, demonstrate \alias's robustness and flexibility in 
% for both seen and unseen goal-directed contact-rich manipulation.
%
% \alias sets up a standardized pipeline for interaction-aware and joint state-action diffusion planning. 
%
We believe its potential to advance the field toward diverse dexterous tasks while remaining grounded in real physics and dynamics.

\vspace{1pt}\noindent\textbf{Future Work} can investigate deployment with hand states sensed and object poses estimated by vision models.

\clearpage
\section*{Acknowledgements}
This paper is partially supported by the General Research Fund of Hong Kong No.17200622 and 17209324, and the Jockey Club STEM Lab of Autonomous Intelligent Systems funded by The Hong Kong Jockey Club Charities Trust. 

{
    \small
    \bibliographystyle{ieeenat_fullname}
    \bibliography{main}

\begin{thebibliography}{58}
\providecommand{\natexlab}[1]{#1}
\providecommand{\url}[1]{\texttt{#1}}
\expandafter\ifx\csname urlstyle\endcsname\relax
  \providecommand{\doi}[1]{doi: #1}\else
  \providecommand{\doi}{doi: \begingroup \urlstyle{rm}\Url}\fi

\bibitem[Ajay et~al.(2023)Ajay, Du, Gupta, Tenenbaum, Jaakkola, and Agrawal]{decisiondiffuser}
Anurag Ajay, Yilun Du, Abhi Gupta, Joshua~B Tenenbaum, Tommi~S Jaakkola, and Pulkit Agrawal.
\newblock Is conditional generative modeling all you need for decision making?
\newblock In \emph{The Eleventh International Conference on Learning Representations}, 2023.

\bibitem[Akkaya et~al.(2019)Akkaya, Andrychowicz, Chociej, Litwin, McGrew, Petron, Paino, Plappert, Powell, Ribas, et~al.]{akkaya2019solving}
Ilge Akkaya, Marcin Andrychowicz, Maciek Chociej, Mateusz Litwin, Bob McGrew, Arthur Petron, Alex Paino, Matthias Plappert, Glenn Powell, Raphael Ribas, et~al.
\newblock Solving rubik's cube with a robot hand.
\newblock \emph{arXiv preprint arXiv:1910.07113}, 2019.

\bibitem[Andrychowicz et~al.(2017)Andrychowicz, Wolski, Ray, Schneider, Fong, Welinder, McGrew, Tobin, Pieter~Abbeel, and Zaremba]{HER}
Marcin Andrychowicz, Filip Wolski, Alex Ray, Jonas Schneider, Rachel Fong, Peter Welinder, Bob McGrew, Josh Tobin, OpenAI Pieter~Abbeel, and Wojciech Zaremba.
\newblock Hindsight experience replay.
\newblock \emph{Advances in neural information processing systems}, 30, 2017.

\bibitem[Andrychowicz et~al.(2020)Andrychowicz, Baker, Chociej, Jozefowicz, McGrew, Pachocki, Petron, Plappert, Powell, Ray, et~al.]{openai2020learning}
OpenAI:~Marcin Andrychowicz, Bowen Baker, Maciek Chociej, Rafal Jozefowicz, Bob McGrew, Jakub Pachocki, Arthur Petron, Matthias Plappert, Glenn Powell, Alex Ray, et~al.
\newblock Learning dexterous in-hand manipulation.
\newblock \emph{The International Journal of Robotics Research}, 39\penalty0 (1):\penalty0 3--20, 2020.

\bibitem[Anthropic(2024)]{claude_3_5}
Anthropic.
\newblock Claude 3.5 sonnet, 2024.
\newblock Available at: \url{https://www.anthropic.com/news/claude-3-5-sonnet}.

\bibitem[Cheang et~al.(2024)Cheang, Chen, Jing, Kong, Li, Li, Liu, Wu, Xu, Yang, et~al.]{GR-2}
Chi-Lam Cheang, Guangzeng Chen, Ya Jing, Tao Kong, Hang Li, Yifeng Li, Yuxiao Liu, Hongtao Wu, Jiafeng Xu, Yichu Yang, et~al.
\newblock Gr-2: A generative video-language-action model with web-scale knowledge for robot manipulation.
\newblock \emph{arXiv preprint arXiv:2410.06158}, 2024.

\bibitem[Chen et~al.(2024{\natexlab{a}})Chen, Mu, Yu, Wei, Wu, Yuan, Liang, Yang, Zhang, Shao, et~al.]{chen2024roboscript}
Junting Chen, Yao Mu, Qiaojun Yu, Tianming Wei, Silang Wu, Zhecheng Yuan, Zhixuan Liang, Chao Yang, Kaipeng Zhang, Wenqi Shao, et~al.
\newblock Roboscript: Code generation for free-form manipulation tasks across real and simulation.
\newblock \emph{arXiv preprint arXiv:2402.14623}, 2024{\natexlab{a}}.

\bibitem[Chen et~al.(2022{\natexlab{a}})Chen, Xu, and Agrawal]{chen2022system}
Tao Chen, Jie Xu, and Pulkit Agrawal.
\newblock A system for general in-hand object re-orientation.
\newblock In \emph{Conference on Robot Learning}, pages 297--307. PMLR, 2022{\natexlab{a}}.

\bibitem[Chen et~al.(2024{\natexlab{b}})Chen, Mu, Liang, Chen, Peng, Chen, Xu, Hu, Zhang, Li, et~al.]{chen2024g3flow}
Tianxing Chen, Yao Mu, Zhixuan Liang, Zanxin Chen, Shijia Peng, Qiangyu Chen, Mingkun Xu, Ruizhen Hu, Hongyuan Zhang, Xuelong Li, et~al.
\newblock G3flow: Generative 3d semantic flow for pose-aware and generalizable object manipulation.
\newblock \emph{arXiv preprint arXiv:2411.18369}, 2024{\natexlab{b}}.

\bibitem[Chen et~al.(2022{\natexlab{b}})Chen, Wu, Wang, Feng, Jiang, Lu, McAleer, Dong, Zhu, and Yang]{chen2022towards}
Yuanpei Chen, Tianhao Wu, Shengjie Wang, Xidong Feng, Jiechuan Jiang, Zongqing Lu, Stephen McAleer, Hao Dong, Song-Chun Zhu, and Yaodong Yang.
\newblock Towards human-level bimanual dexterous manipulation with reinforcement learning.
\newblock \emph{Advances in Neural Information Processing Systems}, 35:\penalty0 5150--5163, 2022{\natexlab{b}}.

\bibitem[Chen et~al.(2023)Chen, Geng, Zhong, Ji, Jiang, Lu, Dong, and Yang]{bidexterous}
Yuanpei Chen, Yiran Geng, Fangwei Zhong, Jiaming Ji, Jiechuang Jiang, Zongqing Lu, Hao Dong, and Yaodong Yang.
\newblock Bi-dexhands: Towards human-level bimanual dexterous manipulation.
\newblock \emph{IEEE Transactions on Pattern Analysis and Machine Intelligence}, 2023.

\bibitem[Chen et~al.(2024{\natexlab{c}})Chen, Chen, Schmid, and Laptev]{chen2024vividex}
Zerui Chen, Shizhe Chen, Cordelia Schmid, and Ivan Laptev.
\newblock Vividex: Learning vision-based dexterous manipulation from human videos.
\newblock \emph{arXiv preprint arXiv:2404.15709}, 2024{\natexlab{c}}.

\bibitem[Chen et~al.(2022{\natexlab{c}})Chen, Van~Wyk, Chao, Yang, Mousavian, Gupta, and Fox]{chen2022dextransfer}
Zoey~Qiuyu Chen, Karl Van~Wyk, Yu-Wei Chao, Wei Yang, Arsalan Mousavian, Abhishek Gupta, and Dieter Fox.
\newblock Dextransfer: Real world multi-fingered dexterous grasping with minimal human demonstrations.
\newblock \emph{arXiv preprint arXiv:2209.14284}, 2022{\natexlab{c}}.

\bibitem[Chi et~al.(2023)Chi, Xu, Feng, Cousineau, Du, Burchfiel, Tedrake, and Song]{diffusion_policy}
Cheng Chi, Zhenjia Xu, Siyuan Feng, Eric Cousineau, Yilun Du, Benjamin Burchfiel, Russ Tedrake, and Shuran Song.
\newblock Diffusion policy: Visuomotor policy learning via action diffusion.
\newblock \emph{The International Journal of Robotics Research}, page 02783649241273668, 2023.

\bibitem[Dhariwal and Nichol(2021)]{diffusion_beatgan}
Prafulla Dhariwal and Alexander Nichol.
\newblock Diffusion models beat gans on image synthesis.
\newblock \emph{Advances in neural information processing systems}, 34:\penalty0 8780--8794, 2021.

\bibitem[Du et~al.(2024)Du, Yang, Dai, Dai, Nachum, Tenenbaum, Schuurmans, and Abbeel]{unipi}
Yilun Du, Sherry Yang, Bo Dai, Hanjun Dai, Ofir Nachum, Josh Tenenbaum, Dale Schuurmans, and Pieter Abbeel.
\newblock Learning universal policies via text-guided video generation.
\newblock \emph{Advances in Neural Information Processing Systems}, 36, 2024.

\bibitem[Feller(2015)]{feller2015theory}
William Feller.
\newblock On the theory of stochastic processes, with particular reference to applications.
\newblock In \emph{Selected Papers I}, pages 769--798. Springer, 2015.

\bibitem[Fu et~al.(2020)Fu, Kumar, Nachum, Tucker, and Levine]{d4rl}
Justin Fu, Aviral Kumar, Ofir Nachum, George Tucker, and Sergey Levine.
\newblock D4rl: Datasets for deep data-driven reinforcement learning.
\newblock \emph{arXiv preprint arXiv:2004.07219}, 2020.

\bibitem[Gupta et~al.(2021)Gupta, Yu, Zhao, Kumar, Rovinsky, Xu, Devlin, and Levine]{gupta2021reset}
Abhishek Gupta, Justin Yu, Tony~Z Zhao, Vikash Kumar, Aaron Rovinsky, Kelvin Xu, Thomas Devlin, and Sergey Levine.
\newblock Reset-free reinforcement learning via multi-task learning: Learning dexterous manipulation behaviors without human intervention.
\newblock In \emph{2021 IEEE International Conference on Robotics and Automation (ICRA)}, pages 6664--6671. IEEE, 2021.

\bibitem[He et~al.(2024)He, Luo, He, Xiao, Zhang, Zhang, Kitani, Liu, and Shi]{he2024omnih2o}
Tairan He, Zhengyi Luo, Xialin He, Wenli Xiao, Chong Zhang, Weinan Zhang, Kris Kitani, Changliu Liu, and Guanya Shi.
\newblock Omnih2o: Universal and dexterous human-to-humanoid whole-body teleoperation and learning.
\newblock \emph{arXiv preprint arXiv:2406.08858}, 2024.

\bibitem[Hinton(2002)]{product_of_expert}
Geoffrey~E Hinton.
\newblock Training products of experts by minimizing contrastive divergence.
\newblock \emph{Neural computation}, 14\penalty0 (8):\penalty0 1771--1800, 2002.

\bibitem[Ho et~al.(2020)Ho, Jain, and Abbeel]{ddpm}
Jonathan Ho, Ajay Jain, and Pieter Abbeel.
\newblock Denoising diffusion probabilistic models.
\newblock \emph{Advances in neural information processing systems}, 33:\penalty0 6840--6851, 2020.

\bibitem[Huang et~al.(2021)Huang, Mordatch, Abbeel, and Pathak]{huang2021generalization}
Wenlong Huang, Igor Mordatch, Pieter Abbeel, and Deepak Pathak.
\newblock Generalization in dexterous manipulation via geometry-aware multi-task learning.
\newblock \emph{arXiv preprint arXiv:2111.03062}, 2021.

\bibitem[Janner et~al.(2022)Janner, Du, Tenenbaum, and Levine]{diffuser}
Michael Janner, Yilun Du, Joshua Tenenbaum, and Sergey Levine.
\newblock Planning with diffusion for flexible behavior synthesis.
\newblock In \emph{International Conference on Machine Learning}, pages 9902--9915. PMLR, 2022.

\bibitem[Kingma and Ba(2015)]{kingma2014adam}
Diederik~P Kingma and Jimmy Ba.
\newblock Adam: A method for stochastic optimization.
\newblock In \emph{International Conference on Learning Representations}, 2015.

\bibitem[Kuznetsov et~al.(2020)Kuznetsov, Shvechikov, Grishin, and Vetrov]{TQC}
Arsenii Kuznetsov, Pavel Shvechikov, Alexander Grishin, and Dmitry Vetrov.
\newblock Controlling overestimation bias with truncated mixture of continuous distributional quantile critics.
\newblock In \emph{International Conference on Machine Learning}, pages 5556--5566. PMLR, 2020.

\bibitem[Liang et~al.(2023{\natexlab{a}})Liang, Huang, Xia, Xu, Hausman, Ichter, Florence, and Zeng]{liang2023code}
Jacky Liang, Wenlong Huang, Fei Xia, Peng Xu, Karol Hausman, Brian Ichter, Pete Florence, and Andy Zeng.
\newblock Code as policies: Language model programs for embodied control.
\newblock In \emph{2023 IEEE International Conference on Robotics and Automation (ICRA)}, pages 9493--9500. IEEE, 2023{\natexlab{a}}.

\bibitem[Liang et~al.(2023{\natexlab{b}})Liang, Mu, Ding, Ni, Tomizuka, and Luo]{adaptdiffuser}
Zhixuan Liang, Yao Mu, Mingyu Ding, Fei Ni, Masayoshi Tomizuka, and Ping Luo.
\newblock Adaptdiffuser: Diffusion models as adaptive self-evolving planners.
\newblock In \emph{International Conference on Machine Learning}, pages 20725--20745. PMLR, 2023{\natexlab{b}}.

\bibitem[Liang et~al.(2024{\natexlab{a}})Liang, Mu, Ma, Tomizuka, Ding, and Luo]{skilldiffuser}
Zhixuan Liang, Yao Mu, Hengbo Ma, Masayoshi Tomizuka, Mingyu Ding, and Ping Luo.
\newblock Skilldiffuser: Interpretable hierarchical planning via skill abstractions in diffusion-based task execution.
\newblock In \emph{Proceedings of the IEEE/CVF Conference on Computer Vision and Pattern Recognition}, pages 16467--16476, 2024{\natexlab{a}}.

\bibitem[Liang et~al.(2024{\natexlab{b}})Liang, Mu, Wang, Ni, Chen, Shao, Zhan, Tomizuka, Luo, and Ding]{liang2024dexdiffuser}
Zhixuan Liang, Yao Mu, Yixiao Wang, Fei Ni, Tianxing Chen, Wenqi Shao, Wei Zhan, Masayoshi Tomizuka, Ping Luo, and Mingyu Ding.
\newblock Dexdiffuser: Interaction-aware diffusion planning for adaptive dexterous manipulation.
\newblock \emph{arXiv preprint arXiv:2411.18562}, 2024{\natexlab{b}}.

\bibitem[Lu et~al.(2022)Lu, Zhou, Bao, Chen, Li, and Zhu]{dpmsolver++}
Cheng Lu, Yuhao Zhou, Fan Bao, Jianfei Chen, Chongxuan Li, and Jun Zhu.
\newblock Dpm-solver++: Fast solver for guided sampling of diffusion probabilistic models.
\newblock \emph{arXiv preprint arXiv:2211.01095}, 2022.

\bibitem[Luo et~al.(2024)Luo, Cao, Christen, Winkler, Kitani, and Xu]{luo2024grasping}
Zhengyi Luo, Jinkun Cao, Sammy Christen, Alexander Winkler, Kris Kitani, and Weipeng Xu.
\newblock Grasping diverse objects with simulated humanoids.
\newblock \emph{arXiv preprint arXiv:2407.11385}, 2024.

\bibitem[Ma et~al.(2023)Ma, Liang, Wang, Huang, Bastani, Jayaraman, Zhu, Fan, and Anandkumar]{eureka}
Yecheng~Jason Ma, William Liang, Guanzhi Wang, De-An Huang, Osbert Bastani, Dinesh Jayaraman, Yuke Zhu, Linxi Fan, and Anima Anandkumar.
\newblock Eureka: Human-level reward design via coding large language models.
\newblock \emph{arXiv preprint arXiv:2310.12931}, 2023.

\bibitem[Mandikal and Grauman(2022)]{mandikal2022dexvip}
Priyanka Mandikal and Kristen Grauman.
\newblock Dexvip: Learning dexterous grasping with human hand pose priors from video.
\newblock In \emph{Conference on Robot Learning}, pages 651--661. PMLR, 2022.

\bibitem[Mu et~al.(2024)Mu, Chen, Zhang, Chen, Yu, Chongjian, Chen, Liang, Hu, Tao, et~al.]{robocodex}
Yao Mu, Junting Chen, Qing-Long Zhang, Shoufa Chen, Qiaojun Yu, GE Chongjian, Runjian Chen, Zhixuan Liang, Mengkang Hu, Chaofan Tao, et~al.
\newblock Robocodex: Multimodal code generation for robotic behavior synthesis.
\newblock In \emph{Forty-first International Conference on Machine Learning}, 2024.

\bibitem[Nagabandi et~al.(2020)Nagabandi, Konolige, Levine, and Kumar]{nagabandi2020deep}
Anusha Nagabandi, Kurt Konolige, Sergey Levine, and Vikash Kumar.
\newblock Deep dynamics models for learning dexterous manipulation.
\newblock In \emph{Conference on Robot Learning}, pages 1101--1112. PMLR, 2020.

\bibitem[Ni et~al.(2023)Ni, Hao, Mu, Yuan, Zheng, Wang, and Liang]{metadiffuser}
Fei Ni, Jianye Hao, Yao Mu, Yifu Yuan, Yan Zheng, Bin Wang, and Zhixuan Liang.
\newblock Metadiffuser: Diffusion model as conditional planner for offline meta-rl.
\newblock In \emph{International Conference on Machine Learning}, pages 26087--26105. PMLR, 2023.

\bibitem[Plappert et~al.(2018)Plappert, Andrychowicz, Ray, McGrew, Baker, Powell, Schneider, Tobin, Chociej, Welinder, Kumar, and Zaremba]{shadow-hand-env}
Matthias Plappert, Marcin Andrychowicz, Alex Ray, Bob McGrew, Bowen Baker, Glenn Powell, Jonas Schneider, Josh Tobin, Maciek Chociej, Peter Welinder, Vikash Kumar, and Wojciech Zaremba.
\newblock Multi-goal reinforcement learning: Challenging robotics environments and request for research, 2018.

\bibitem[Puterman(1994)]{puterman1994markov}
Martin~L Puterman.
\newblock \emph{Markov decision processes: discrete stochastic dynamic programming}.
\newblock John Wiley \& Sons, 1994.

\bibitem[Qin et~al.(2022)Qin, Wu, Liu, Jiang, Yang, Fu, and Wang]{qin2022dexmv}
Yuzhe Qin, Yueh-Hua Wu, Shaowei Liu, Hanwen Jiang, Ruihan Yang, Yang Fu, and Xiaolong Wang.
\newblock Dexmv: Imitation learning for dexterous manipulation from human videos.
\newblock In \emph{European Conference on Computer Vision}, pages 570--587. Springer, 2022.

\bibitem[Rajeswaran et~al.(2017{\natexlab{a}})Rajeswaran, Kumar, Gupta, Vezzani, Schulman, Todorov, and Levine]{adroit}
Aravind Rajeswaran, Vikash Kumar, Abhishek Gupta, Giulia Vezzani, John Schulman, Emanuel Todorov, and Sergey Levine.
\newblock Learning complex dexterous manipulation with deep reinforcement learning and demonstrations.
\newblock \emph{arXiv preprint arXiv:1709.10087}, 2017{\natexlab{a}}.

\bibitem[Rajeswaran et~al.(2017{\natexlab{b}})Rajeswaran, Kumar, Gupta, Vezzani, Schulman, Todorov, and Levine]{rajeswaran2017learning}
Aravind Rajeswaran, Vikash Kumar, Abhishek Gupta, Giulia Vezzani, John Schulman, Emanuel Todorov, and Sergey Levine.
\newblock Learning complex dexterous manipulation with deep reinforcement learning and demonstrations.
\newblock \emph{arXiv preprint arXiv:1709.10087}, 2017{\natexlab{b}}.

\bibitem[Ronneberger et~al.(2015)Ronneberger, Fischer, and Brox]{ronneberger2015u}
Olaf Ronneberger, Philipp Fischer, and Thomas Brox.
\newblock U-net: Convolutional networks for biomedical image segmentation.
\newblock In \emph{Medical Image Computing and Computer-Assisted Intervention--MICCAI 2015: 18th International Conference, Munich, Germany, October 5-9, 2015, Proceedings, Part III 18}, pages 234--241. Springer, 2015.

\bibitem[{Shadow Robot Company}(2024)]{shadowrobot}
{Shadow Robot Company}.
\newblock Shadow robot.
\newblock \url{https://www.shadowrobot.com/}, 2024.
\newblock Accessed: 2024-11-14.

\bibitem[Sivakumar et~al.(2022)Sivakumar, Shaw, and Pathak]{sivakumar2022robotic}
Aravind Sivakumar, Kenneth Shaw, and Deepak Pathak.
\newblock Robotic telekinesis: Learning a robotic hand imitator by watching humans on youtube.
\newblock \emph{arXiv preprint arXiv:2202.10448}, 2022.

\bibitem[Todorov et~al.(2012)Todorov, Erez, and Tassa]{mujoco}
Emanuel Todorov, Tom Erez, and Yuval Tassa.
\newblock Mujoco: A physics engine for model-based control.
\newblock In \emph{2012 IEEE/RSJ international conference on intelligent robots and systems}, pages 5026--5033. IEEE, 2012.

\bibitem[Wan et~al.(2023)Wan, Geng, Liu, Shan, Yang, Yi, and Wang]{wan2023unidexgrasp++}
Weikang Wan, Haoran Geng, Yun Liu, Zikang Shan, Yaodong Yang, Li Yi, and He Wang.
\newblock Unidexgrasp++: Improving dexterous grasping policy learning via geometry-aware curriculum and iterative generalist-specialist learning.
\newblock In \emph{Proceedings of the IEEE/CVF International Conference on Computer Vision}, pages 3891--3902, 2023.

\bibitem[Wang et~al.(2024)Wang, Shi, Wang, Zhang, Fei-Fei, and Liu]{wang2024dexcap}
Chen Wang, Haochen Shi, Weizhuo Wang, Ruohan Zhang, Li Fei-Fei, and C~Karen Liu.
\newblock Dexcap: Scalable and portable mocap data collection system for dexterous manipulation.
\newblock \emph{arXiv preprint arXiv:2403.07788}, 2024.

\bibitem[Weisstein(2002)]{weisstein2002heaviside}
Eric~W Weisstein.
\newblock Heaviside step function.
\newblock \emph{https://mathworld. wolfram. com/}, 2002.

\bibitem[Weng et~al.(2024)Weng, Lu, Kragic, and Lundell]{weng2024dexdiffuser}
Zehang Weng, Haofei Lu, Danica Kragic, and Jens Lundell.
\newblock Dexdiffuser: Generating dexterous grasps with diffusion models.
\newblock \emph{arXiv preprint arXiv:2402.02989}, 2024.

\bibitem[Wu et~al.(2024{\natexlab{a}})Wu, Ge, Guo, Wang, Liang, Lu, Shan, and Luo]{wu2024plot2code}
Chengyue Wu, Yixiao Ge, Qiushan Guo, Jiahao Wang, Zhixuan Liang, Zeyu Lu, Ying Shan, and Ping Luo.
\newblock Plot2code: A comprehensive benchmark for evaluating multi-modal large language models in code generation from scientific plots.
\newblock \emph{arXiv preprint arXiv:2405.07990}, 2024{\natexlab{a}}.

\bibitem[Wu et~al.(2024{\natexlab{b}})Wu, Gan, Wu, Cheng, Yang, Zhu, and Dong]{wu2024unidexfpm}
Tianhao Wu, Yunchong Gan, Mingdong Wu, Jingbo Cheng, Yaodong Yang, Yixin Zhu, and Hao Dong.
\newblock Unidexfpm: Universal dexterous functional pre-grasp manipulation via diffusion policy.
\newblock \emph{arXiv preprint arXiv:2403.12421}, 2024{\natexlab{b}}.

\bibitem[Wu and He(2018)]{wu2018group}
Yuxin Wu and Kaiming He.
\newblock Group normalization.
\newblock In \emph{Proceedings of the European conference on computer vision (ECCV)}, pages 3--19, 2018.

\bibitem[Xie et~al.(2024)Xie, Zhao, Wu, Liu, Luo, Zhong, Yang, and Yu]{text2reward}
Tianbao Xie, Siheng Zhao, Chen~Henry Wu, Yitao Liu, Qian Luo, Victor Zhong, Yanchao Yang, and Tao Yu.
\newblock Text2reward: Reward shaping with language models for reinforcement learning.
\newblock In \emph{The Twelfth International Conference on Learning Representations}, 2024.

\bibitem[Yu and Wang(2022)]{yu2022dexterous}
Chunmiao Yu and Peng Wang.
\newblock Dexterous manipulation for multi-fingered robotic hands with reinforcement learning: A review.
\newblock \emph{Frontiers in Neurorobotics}, 16:\penalty0 861825, 2022.

\bibitem[Zhao et~al.(2023)Zhao, Kumar, Levine, and Finn]{ACT}
Tony~Z Zhao, Vikash Kumar, Sergey Levine, and Chelsea Finn.
\newblock Learning fine-grained bimanual manipulation with low-cost hardware.
\newblock \emph{arXiv preprint arXiv:2304.13705}, 2023.

\bibitem[Zhou et~al.(2024)Zhou, Yuan, Fu, and Lu]{zhou2024learning}
Bohan Zhou, Haoqi Yuan, Yuhui Fu, and Zongqing Lu.
\newblock Learning diverse bimanual dexterous manipulation skills from human demonstrations.
\newblock \emph{arXiv preprint arXiv:2410.02477}, 2024.

\bibitem[Zhu et~al.(2019)Zhu, Gupta, Rajeswaran, Levine, and Kumar]{RLdexterousReview}
Henry Zhu, Abhishek Gupta, Aravind Rajeswaran, Sergey Levine, and Vikash Kumar.
\newblock Dexterous manipulation with deep reinforcement learning: Efficient, general, and low-cost.
\newblock In \emph{2019 International Conference on Robotics and Automation (ICRA)}, pages 3651--3657. IEEE, 2019.

\end{thebibliography}
}

\clearpage
\appendix
\setcounter{page}{1}
\maketitlesupplementary

\section{Brief Theoretical Review of Gradient Guidance in Classifier-guided Diffusion Model}
\label{append:math}
For a trajectory $\btau$, we define the reverse process of a standard diffusion model as $p_{\theta}(\btau^i|\btau^{i+1})$. To enable goal-directed generation, we introduce a classifier $p_{\phi}(\by|\btau^i)$ that evaluates whether a noisy trajectory $\btau^i$ satisfies the goal condition $\by$. The combined process is denoted as $p_{\theta,\phi}(\btau^i|\btau^{i+1},\by)$.

Under property of Markov process in diffusion model illustrated by~\cite{diffusion_beatgan,adaptdiffuser}, we can establish:
\begin{equation}
p_{\theta, \phi}\left(\by \mid \btau^{i}, \btau^{i+1}\right) = p_{\phi}\left(\by \mid \btau^{i}\right).
\end{equation}

This leads to our first key theorem:
\begin{theorem}
The conditional sampling probability of the reverse diffusion process $p_{\theta,\phi}(\btau^i \mid \btau^{i+1},\by)$ can be decomposed into a product of the unconditional transition probability $p_{\theta}(\btau^i \mid \btau^{i+1})$ and the classifier probability $p_{\phi}(\by \mid \btau^i)$, up to a normalizing constant $Z$:
\begin{equation}
p_{\theta,\phi}(\btau^i \mid \btau^{i+1},\by) = Z p_{\theta}(\btau^i \mid \btau^{i+1})p_{\phi}(\by \mid \btau^i).
\label{eq:theorem1}
\end{equation}
\end{theorem}
\begin{proof}
By applying Bayes' theorem:
\begin{equation*}
\begin{split}
p_{\theta,\phi}(\btau^i \mid &\btau^{i+1},\ \by) = \frac{p_{\theta, \phi}\left(\btau^{i}, \btau^{i+1}, \by\right)}{p_{\theta, \phi}\left(\btau^{i+1}, \by\right)}
\\
& =\frac{p_{\theta, \phi}\left(\by \mid \btau^{i}, \btau^{i+1}\right) p_{\theta}\left(\btau^{i}, \btau^{i+1}\right)}{p_{\phi}\left(\by \mid \btau^{i+1}\right) p_{\theta}\left(\btau^{i+1}\right)}
\\
& =\frac{p_{\theta, \phi}\left(\by \mid \btau^{i}, \btau^{i+1}\right) p_{\theta}\left(\btau^{i} \mid \btau^{i+1}\right) p_{\theta}\left(\btau^{i+1}\right)}{p_{\phi}\left(\by \mid \btau^{i+1}\right) p_{\theta}\left(\btau^{i+1}\right)}
\\
& =\frac{p_{\phi}\left(\by \mid \btau^{i}\right) p_{\theta}\left(\btau^{i} \mid \btau^{i+1}\right)}{p_{\phi}\left(\by \mid \btau^{i+1}\right)},
\end{split}
\label{eq:guided-diff}
\end{equation*}
where $p_{\phi}\left(\by \mid \btau^{i+1}\right)$ becomes the normalizing constant $Z$ as it is independent of $\btau^i$.
\end{proof}

For practical implementation, we derive:
\begin{theorem}
Under the assumption of sufficient reverse diffusion steps, the conditional sampling probability $p_{\theta,\phi}(\btau^i|\btau^{i+1},\by)$ can be approximated by a modified Gaussian distribution, where the mean is shifted by the classifier gradient and the variance remains unchanged from the unconditional process:
\begin{equation}
% \small
p_{\theta,\phi}(\btau^i|\btau^{i+1},\by) \approx \mathcal{N}(\btau^i; \mu_{\theta} + \Sigma \nabla_{\btau} \log p_{\phi}\left(\by \mid \btau^{i}\right), \Sigma),
\end{equation}
where $\mu_{\theta}$ and $\Sigma$ denote the mean and variance of the unconditional reverse diffusion process $p_{\theta}(\btau^i \mid \btau^{i+1})$.
\end{theorem}
\begin{proof}
First, express the unconditional process as:
\begin{align*}
p_{\theta}(\btau^i \mid \btau^{i+1}) &= \mathcal{N}(\btau^i; \mu_{\theta}, \Sigma). 
\\
\log p_{\theta}(\btau^i \mid \btau^{i+1}) &= -\frac{1}{2}(\btau^i - \mu_{\theta})^T \Sigma^{-1} (\btau^{i} - \mu_\theta) + C.
\end{align*}
Apply Taylor expansion to $\log p_{\phi}\left(\by \mid \btau^{i}\right)$ around $\btau^i=\mu_{\theta}$:
\begin{align*}
\log p_{\phi}\left(\by \mid \btau^{i}\right) &= \log p_{\phi}\left(\by \mid \btau^{i}\right)|_{\btau^{i}=\mu_{\theta}}\\
&+\left.\left(\btau^{i}-\mu_{\theta}\right) \nabla_{\btau^{i}} \log p_{\phi}\left(\by \mid \btau^{i}\right)\right|_{\btau^{i}=\mu_{\theta}}.
\end{align*}
Applying the logarithm to both sides of Eq.~\ref{eq:theorem1}:
\begin{equation*}
\begin{split}
\log p_{\theta,\phi}(\btau^i |\btau^{i+1},\by) &= \log p_{\theta}(\btau^i|\btau^{i+1}) + \log p_{\phi}(\by|\btau^i)+C_1\\
&= -\frac{1}{2}\left(\btau^{i}-\mu_{\theta}\right)^{T} \Sigma^{-1}\left(\btau^{i}-\mu_{\theta}\right)\\
&\ \ \ \ \ \ \ +\left(\btau^{i}-\mu_{\theta}\right) \nabla \log p_{\phi}\left(\by \mid \btau^{i}\right) +C_{2}
\end{split}
\end{equation*}
Completing the square yields:
\begin{equation*}
\begin{split}
RHS =-&\frac{1}{2}\left(\btau^{i}-\mu_{\theta}-\Sigma \nabla \log p_{\phi}\left(\by \mid \btau^{i}\right)\right)^{T} \Sigma^{-1}\\
&\times\left(\btau^{i}-\mu_{\theta}-\Sigma \nabla \log p_{\phi}\left(\by \mid \btau^{i}\right)\right)+C_{3}.
\end{split}
\end{equation*}

This establishes the Gaussian form of the approximation.
\end{proof}

This theoretical framework underlies our goal-directed diffusion planning approach.

\section{Environment Settings}
\label{append:environ}
The door task represents multi-stage manipulation where the hand must reach and rotate a door handle, then pull or push the door to a target angle. The hammer task tests tool use capabilities, requiring the hand to grasp the hammer and strike a nail, while the pen and the block task evaluates in-hand dexterity, targeting continuous object reorientation. 
And the object relocation task requires to grasp the ball first and then move to the desired position.

\section{More Visualizations}

Different from concurrent work~\cite{weng2024dexdiffuser} that focuses on grasping tasks, we conduct experiments on challenging dexterous manipulation benchmarks including door, pen, hammer, and block manipulation tasks, which require sophisticated contact-rich interactions and precise goal-directed control.

\subsection{Goal Adaptive Door Tasks}

\begin{figure*}[tb]
  \centering
   \includegraphics[width=0.99\linewidth]{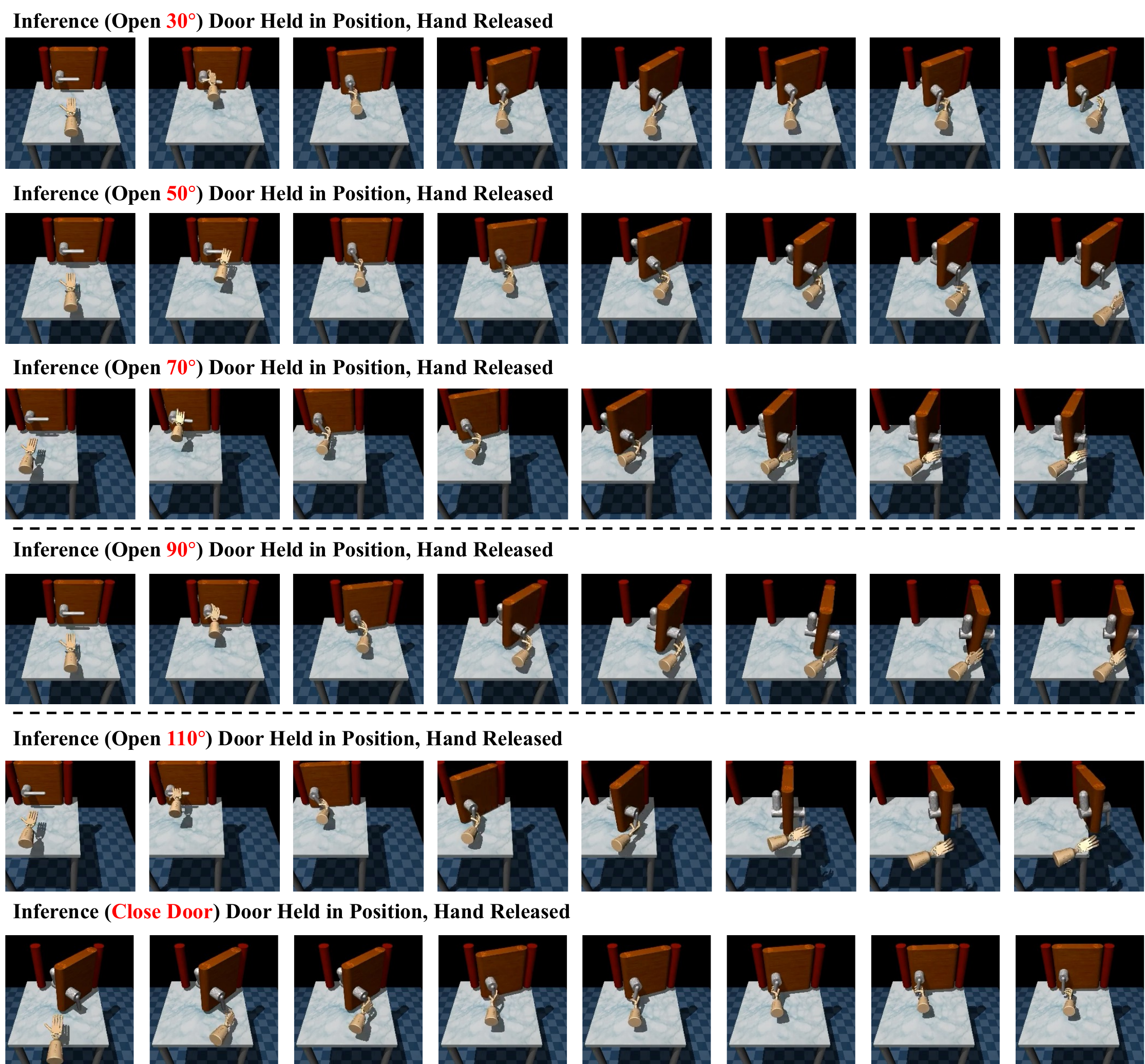}
   \vspace{-5pt}
   \caption{\textbf{Visualization of goal-adaptive door manipulation.} Despite training only on 90$^\circ$ demonstrations, \alias adapts to various target angles (30$^\circ$-110$^\circ$) and door closing, maintaining stable control and physical consistency throughout the motion sequence.}
   \label{fig:vis_door}
   \vspace{-10pt}
\end{figure*}

We present detailed visualizations of \alias's performance on various door manipulation tasks in Fig.~\ref{fig:vis_door}, demonstrating its adaptability to different target angles and even task reversal. Each row shows a sequence of eight frames capturing key moments in the manipulation process.

For opening tasks with different target angles, we observe consistent behavior patterns: the hand first approaches and grasps the handle, then rotates it precisely to the specified angle, and finally releases while maintaining the door's position. Notably, even though trained only on 90$^\circ$ demonstrations, \alias successfully generalizes to both smaller angles (30$^\circ$, 50$^\circ$, 70$^\circ$) and a larger angle (110$^\circ$), maintaining stable control throughout the motion.

The final row demonstrates the model's capability for task reversal - closing the door. This is particularly challenging as it requires adapting the learned manipulation strategy in the opposite direction. The sequence shows the hand approaching the open door, grasping the handle, and smoothly guiding it to the closed position.

Across all variations, we observe several key characteristics: (1) Consistent contact-rich interaction phases; (2) Precise angle control regardless of target; (3) Stable door holding after reaching the target; (4) Smooth hand retraction while maintaining door position.

These visualizations illustrate \alias's robust goal adaptation capabilities while maintaining physical realism in the manipulation process.

\begin{figure*}[tb]
  \centering
   \includegraphics[width=0.99\linewidth]{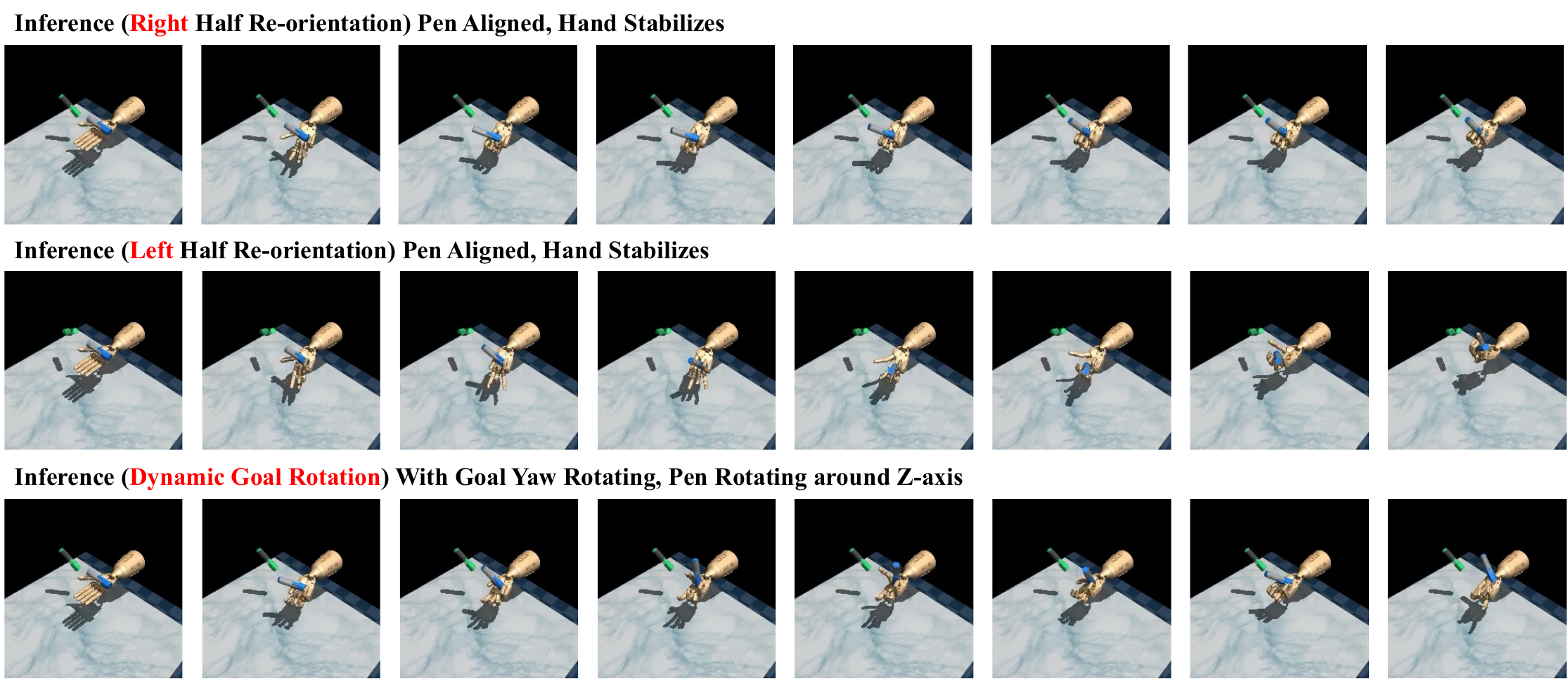}
   \vspace{-7pt}
   \caption{\textbf{Visualization of pen manipulation tasks.} Top: right-half re-orientation (training distribution). Middle: left-half re-orientation, requiring challenging large-arc rotation from the initial horizontal-right position. Bottom: dynamic goal tracking where \textbf{target yaw angle rotates uniformly}, demonstrating the model's ability to generalize from static to dynamic goals.}
   \label{fig:vis_pen}
   \vspace{-8pt}
\end{figure*}

\begin{figure*}[tb]
  \centering
   \includegraphics[width=0.99\linewidth]{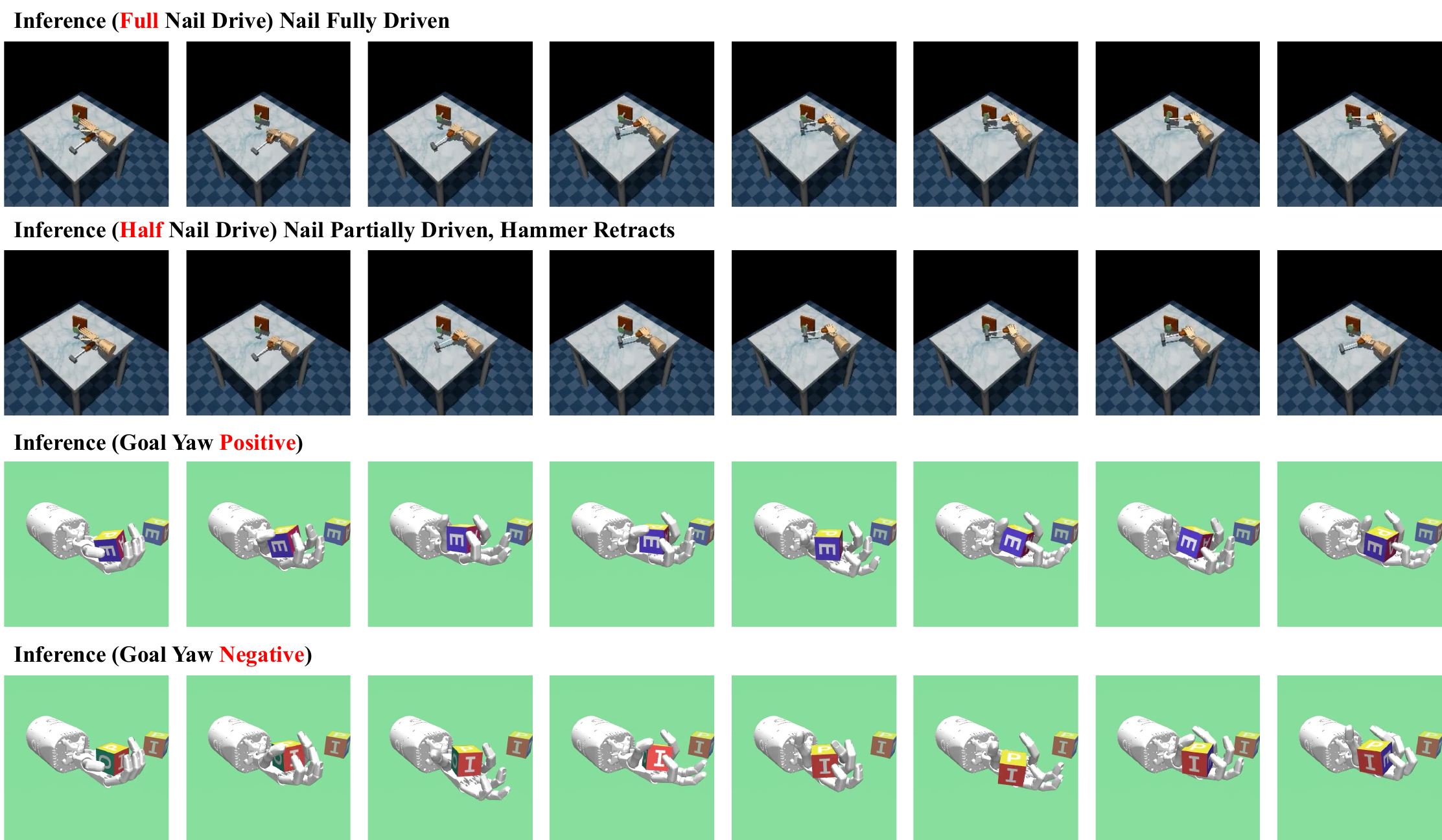}
   \vspace{-5pt}
   \caption{\textbf{Visualization of hammer and block manipulation tasks.} Top two rows: full and partial nail-driving tasks, demonstrating precise control over interaction depth. Bottom two rows: block orientation tasks with quaternion-based pose control, showing adaptation to both positive and negative yaw rotations while maintaining multi-angle alignment.}
   \label{fig:vis_hammer_block}
   \vspace{-8pt}
\end{figure*}

\subsection{Other Dexterous Manipulation Tasks}

First, we showcase our model's capabilities on pen manipulation tasks with detailed visualizations, in Fig.~\ref{fig:vis_pen}. The first two rows demonstrate the model's performance on standard re-orientation tasks: right-half and left-half re-orientation. Notably, as the pen starts from a horizontal-right position, the left-half re-orientation (second row) is particularly challenging, requiring a large rotational arc of nearly 180 degrees to reach the target orientation in the left hemisphere.

Beyond these static goal tasks, we further evaluate our model's adaptability through a dynamic goal rotation task (third row). Using the model trained on full re-orientation data, we design a scenario where the target orientation's \textit{yaw} angle uniformly rotates over time. The visualization demonstrates that our model successfully learns the underlying rotational dynamics \textit{around the z-axis}, smoothly tracking the time-varying target while maintaining stable manipulation.

For the hammer task in Fig.~\ref{fig:vis_hammer_block}, we demonstrate both full and partial nail-driving capabilities. The first row shows the complete nail-driving sequence, where the hand grasps the hammer, positions it precisely, and drives the nail fully into the board. The second row showcases our partial driving task, where the model exhibits precise control by stopping halfway and smoothly retracting the hammer, demonstrating fine-grained control over the manipulation process.

For the block manipulation task also in Fig.~\ref{fig:vis_hammer_block}, we present two scenarios of quaternion-based orientation control. In the first sequence (Goal Yaw Positive), the hand needs to carefully adjust multiple rotational degrees of freedom to achieve the target pose, as the task requires alignment in all three orientation angles. The second sequence (Goal Yaw Negative) presents a more challenging scenario, requiring a larger rotational motion around the z-axis while maintaining control over other orientation angles. This demonstrates our model's capability to handle complex, multi-dimensional orientation targets in quaternion space.

\section{Implementation Details}
We implement our framework following standard diffusion model settings~\cite{diffuser} with several modifications:

\myparagraph{Network Architecture.} We adopt a temporal U-Net~\cite{ronneberger2015u} architecture consisting of 6 residual blocks for noise prediction. Each block contains dual temporal convolutions with group normalization~\cite{wu2018group}, followed by a Mish activation~\cite{wu2018group}. Timestep information is injected through a linear embedding layer and added after the first convolution in each block. The dynamics model uses a 3-layer MLP with batch normalization, ReLU activation, and hidden dimension $512$.

\myparagraph{Training Configuration.} The model is optimized using Adam~\cite{kingma2014adam} optimizer with a learning rate of $2\times10^{-4}$ and batch size $256$, trained for $5\times10^{5}$ steps across all tasks. For both our method and the classifier-free baselines~\cite{decisiondiffuser,diffusion_policy}, we predict the denoised trajectory $\btau_{0}$ directly rather than the noise term $\epsilon$, which is incentive to the performance of classifier-free methods.

\myparagraph{Task-Specific Parameters.} We use different planning horizons during training ($T=32$) and inference ($T=8$ for door / block tasks, $T=32$ for hammer / pen tasks). The diffusion process uses $K=20$ denoising steps across tasks. 

The guidance scale $\alpha$ is task-dependent, selected from $\{500, 1000, 2000\}$ based on empirical performance.

\myparagraph{Computational Resources.} All models are trained on a single NVIDIA GeForce RTX 3090 GPU, requiring training for approximately $30$ hours per task.

\section{LLM-based Guidance Generation Prompts}
\label{appendix:prompts}

\subsection{Overview}
We present our structured prompting strategy for generating guidance functions through LLMs, which can be abstracted by the experts who developed the environment. Our prompts comprise several key components:

\vspace{3pt}\myparagraph{Expert Role Definition.} We begin by defining the LLM's role as an expert in robotics, diffusion models, and code generation, specifically focusing on developing guidance functions for diffusion-based planners.

\vspace{3pt}\myparagraph{Environment Abstraction.} The environment is represented through a comprehensive class hierarchy:
\begin{itemize}
    \item BaseEnv: Contains core components (hand, objects) and observation space definition;
    \item AdroitHand: Detailed 28-DOF joint specification;
    \item Supporting Classes: Door, Handle, \etc, with physical properties and state representations.
\end{itemize}

\vspace{3pt}\myparagraph{Technical Context.} We provide three essential contexts:
\begin{itemize}
    \item Interaction Knowledge: Defines dual-phase guidance strategy (pre-interaction and post-interaction);
    \item Function Call Paradigms: Specifies normalization handling and dynamics model usage through function call;
    \item Differentiability Requirements: Ensures differentiability, proper tensor operations, and physical consistency.
\end{itemize}

\vspace{3pt}\myparagraph{Generation Hints.} We include:
\begin{itemize}
    \item Task Instruction;
    \item Task-specific constraints and requirements;
    \item (Optional) Few-shot examples demonstrating specific techniques like soft interpolation and reward scaling.
\end{itemize}

\vspace{3pt}\noindent \textbf{From next page}, we provide the complete prompt templates used for generating guidance functions.

% \clearpage
\onecolumn

\subsection{Hand Door Task Prompt Example}
\begin{promptcode}[Hand Door Task Prompt Example]
You are an expert in robotics, diffusion model, reinforcement learning, and code generation.
We are going to use an Adroit Shadow Hand to complete given tasks. The action space of the robot is a normalized `Box(-1.0, 1.0, (28,), float32)`. 

Now I want you to help me write a guidance function for a diffusion-based planner. 
1. The guidance function is used to steer the sampling process toward desired outcomes during the reverse diffusion process. 
2. The guidance function should be differentiable, which computes a scalar reward indicating how well each intermediate trajectory aligns with the task objectives.

In manipulation tasks involving interaction with an object, such as opening a door, hammer striking, note that we cannot directly control the object's state. Thus, the guidance function should consider a two-phase approach:
Phase 1 (Pre-Interaction Phase): The guidance function should focus solely on guiding the hand's state to align with the object's handle or interaction point.
Phase 2 (Post-Interaction Phase): Once the hand is in contact with the object, the guidance function should aim to move the object towards achieving the task goal. During this phase, the guidance function typically include the following components (some part is optional, so only include them if really necessary):
1. difference between the current state of the object and its goal state
2. dynamics constraints to ensure the interactions between the hand and the object are physically plausible
3. regularization of the object's state change (e.g., limiting the hinge state change of a door to avoid abrupt movements).
4. [optional] extra constraint of the target object, which is often implied by the task instruction
5. [optional] extra constraint of the robot, which is often implied by the task instruction
...

/*\textbf{Environment Description:}*/
class BaseEnv(gym.Env):
    self.hand : AdroitHand     # The Adroit Shadow Hand used in the environment
    self.door : Door           # The Door object in the environment
    self.dt : float            # The time between two actions, in seconds

    def get_obs(self) -> np.ndarray[(30,)]:
        # Returns the observation vector
        obs = np.concatenate([
            self.hand.get_joint_positions(),                      # Indices 0-27
            [self.door.hinge.angle],                              # Index 28
            [self.door.latch.angle],                              # Index 29
            self.hand.palm.get_position()                         # Indices 30-32
            self.door.handle.get_position()                       # Indices 33-35
        ])
        return obs

class AdroitHand:
    self.arm : Arm             # The arm component of the hand
    self.wrist : Wrist         # The wrist component of the hand
    self.fingers : Fingers     # The fingers of the hand
    self.palm : Palm           # The palm of the hand

    def get_joint_positions(self) -> np.ndarray[(28,)]:
        # Returns the angular positions of all joints in the hand and arm
        return np.array([
            self.arm.translation_z.position,      # Index 0: ARTz
            self.arm.rotation_x.angle,            # Index 1: ARRx
            self.arm.rotation_y.angle,            # Index 2: ARRy
            self.arm.rotation_z.angle,            # Index 3: ARRz
            self.wrist.wrist_joint_1.angle,       # Index 4: WRJ1
            self.wrist.wrist_joint_0.angle,       # Index 5: WRJ0
            # Finger joints
            self.fingers.ffj3.angle,              # Index 6: FFJ3
            self.fingers.ffj2.angle,              # Index 7: FFJ2
            self.fingers.ffj1.angle,              # Index 8: FFJ1
            self.fingers.ffj0.angle,              # Index 9: FFJ0
            self.fingers.mfj3.angle,              # Index 10: MFJ3
            self.fingers.mfj2.angle,              # Index 11: MFJ2
            self.fingers.mfj1.angle,              # Index 12: MFJ1
            self.fingers.mfj0.angle,              # Index 13: MFJ0
            self.fingers.rfj3.angle,              # Index 14: RFJ3
            self.fingers.rfj2.angle,              # Index 15: RFJ2
            self.fingers.rfj1.angle,              # Index 16: RFJ1
            self.fingers.rfj0.angle,              # Index 17: RFJ0
            self.fingers.lfj4.angle,              # Index 18: LFJ4
            self.fingers.lfj3.angle,              # Index 19: LFJ3
            self.fingers.lfj2.angle,              # Index 20: LFJ2
            self.fingers.lfj1.angle,              # Index 21: LFJ1
            self.fingers.lfj0.angle,              # Index 22: LFJ0
            self.fingers.thj4.angle,              # Index 23: THJ4
            self.fingers.thj3.angle,              # Index 24: THJ3
            self.fingers.thj2.angle,              # Index 25: THJ2
            self.fingers.thj1.angle,              # Index 26: THJ1
            self.fingers.thj0.angle               # Index 27: THJ0
        ])

class Arm:
    self.translation_z : SlideJoint  # ARTz
    self.rotation_x : HingeJoint     # ARRx
    self.rotation_y : HingeJoint     # ARRy
    self.rotation_z : HingeJoint     # ARRz

class Wrist:
    self.wrist_joint_1 : HingeJoint  # WRJ1
    self.wrist_joint_0 : HingeJoint  # WRJ0

class Fingers:
    # Forefinger joints
    self.ffj3 : HingeJoint  # FFJ3
    self.ffj2 : HingeJoint  # FFJ2
    self.ffj1 : HingeJoint  # FFJ1
    self.ffj0 : HingeJoint  # FFJ0

    # Middle finger joints
    self.mfj3 : HingeJoint  # MFJ3
    self.mfj2 : HingeJoint  # MFJ2
    self.mfj1 : HingeJoint  # MFJ1
    self.mfj0 : HingeJoint  # MFJ0

    # Ring finger joints
    self.rfj3 : HingeJoint  # RFJ3
    self.rfj2 : HingeJoint  # RFJ2
    self.rfj1 : HingeJoint  # RFJ1
    self.rfj0 : HingeJoint  # RFJ0

    # Little finger joints
    self.lfj4 : HingeJoint  # LFJ4
    self.lfj3 : HingeJoint  # LFJ3
    self.lfj2 : HingeJoint  # LFJ2
    self.lfj1 : HingeJoint  # LFJ1
    self.lfj0 : HingeJoint  # LFJ0

    # Thumb joints
    self.thj4 : HingeJoint  # THJ4
    self.thj3 : HingeJoint  # THJ3
    self.thj2 : HingeJoint  # THJ2
    self.thj1 : HingeJoint  # THJ1
    self.thj0 : HingeJoint  # THJ0

class Palm:
    self.pose : ObjectPose         # The 3D position and orientation of the palm

    def get_position(self) -> np.ndarray[(3,)]:
        # Returns the position of the palm in world coordinates
        return self.pose.position

class Door:
    self.latch : HingeJoint        # The latch joint of the door
    self.hinge : HingeJoint        # The hinge joint of the door
    self.handle : Handle           # The handle of the door

class Handle:
    self.pose : ObjectPose         # The 3D position and orientation of the handle

    def get_position(self) -> np.ndarray[(3,)]:
        # Returns the position of the handle in world coordinates
        return self.pose.position

class HingeJoint:
    self.angle : float                 # Joint angle in radians
    self.angular_velocity : float      # Joint angular velocity in radians per second

class SlideJoint:
    self.position : float              # Position along the slide in meters
    self.velocity : float              # Velocity along the slide in meters per second

class ObjectPose:
    self.position : np.ndarray[(3,)]    # 3D position in world coordinates
    self.orientation : np.ndarray[(4,)] # Quaternion orientation (w, x, y, z)

Observation Index Mapping:
Index 0: Linear translation of the full arm towards the door (self.hand.arm.translation_z.position);
Index 1-27: Angular positions of the hand and arm joints (as per the joint order above);
Index 28: Angular position of the door hinge (self.door.hinge.angle);
Index 29: Angular position of the door latch (self.door.latch.angle);
Index 30-32: Position of the center of the palm in x, y, z (self.hand.palm.get_position());
Index 33-35: Position of the handle of the door in x, y, z (self.door.handle.get_position()).

/*\textbf{Additional knowledge:}*/
1. All angles are expressed in radians.
2. The input `normed_obs` is a tensor with shape (B, H, obs_dim), `normed_actions` is a tensor with shape (B, H, act_dim), where B is the batch size, H is the horizon length. The normed_obs is gotten from `normed_obs = get_obs()`.
3. If you need to match the observations or actions to some explicit value and if not without_normalizer, you should unnormalize them using `self.unnormalize(normed_obs, is_obs=True)`.
4. If `dyn_model` is provided, please call `self.cal_dyn_reward(state=normed_obs, action=normed_actions)` to calculates the reward for dynamics inconsistency (a scalar value) between generated states and actions. Only consider it in phase 2. Pay attention the input should be normed_obs and normed_actions before unnormalizing them.
5. Use L2 distance via `torch.norm(,p=2)` to calculate all the difference instead of mse loss or `torch.abs`.
6. The transition between Phase 1 and Phase 2 by using a grasp mask to determine if the hand has successfully grasped the object. Use a condition like `mask = torch.norm(palm_pos[:, 0, :] - handle_pos[:, 0, :], p=2, dim=1) < 0.1` to switch from guiding only the hand to guiding both the hand and the object.

You are allowed to use any existing Python package if applicable, but only use them when absolutely necessary. Please import the required packages at the beginning of the function. 

/*\textbf{I want it to fulfill the following task: \textcolor{red}{\{"Write a guidance function for a diffusion-based planner that helps the Adroit Shadow Hand open the door to 30 degrees (pi/6 radians)."\}}}*/
1. Please think step by step and explain what it means in the context of this environment;
2. Then write a differentiable guidance function that guides the planner to generate actions smoothly based on the current normed state and action, with the function prototype as `def guidance_fn(self, normed_obs, normed_actions, dyn_model=None, without_normalizer=False)`. The function should return the `reward` as a torch.Tensor of shape `(B,)`;
3. Make sure the guidance aligns with the two phases: In Phase 1, only calculate a pre-grasp reward to guide the hand closer to the object. In Phase 2, guide both the object toward the final task goal. Ensure object velocity constraints are applied to regulate object state changes.
4. All the reward including the goal achieving reward should be across all horizon steps. For some term, use `torch.mean()` to accumulate reward over the horizon. For terms where the last dimension is 1 (such as angles), we should use torch.squeeze to remove that dimension before calculating the norm at dimension 1, rather than dimension 2.
5. Use `self.scaling_factors` as an empty dictionary by default. If the scaling factor for any reward component does not exist, initialize it adaptively to make that first reward term in batch approximately 12 initially, except for the goal-achieving reward (make the reward 30) and the dynamics reward (make it 1.2).
6. Take care of variables' type, never use functions or variables not provided. Ensure that all operations are compatible with PyTorch tensors and the function is differentiable. Do not use any absolute value operation and inplace operations, e.g. `x += 1`, `x[0] = 1`, using `x = x + 1` instead. 
7. Pay attention to the physical meaning of each dimension in the observation and action data as explained in the environment description above.
8. When you writing code, you can also add some comments as your thought, like this:
```
# Here unnormalize the observations if a normalizer is provided
# Here use `torch.norm` to compute the L2 distance between the current and target angles for the door hinge
# Here cauculate the grasp mask for the pre-interaction phase
```

/*\textbf{Few-shot hint:}*/
1. Ensure that the guidance function uses soft interpolation for targets, e.g., smoothly guiding the door hinge angle towards soft goals over the trajectory horizon like `interpolated_angle = (1 - alpha) * current_angle + alpha * target_angle`. 
\end{promptcode}

\subsection{Hand Pen Task Prompt Example}
\begin{promptcode}[Hand Pen Task Prompt Example]
You are an expert in robotics, diffusion model, reinforcement learning, and code generation.
We are going to use an Adroit Shadow Hand to complete given tasks. The action space of the robot is a normalized `Box(-1.0, 1.0, (28,), float32)`. 

Now I want you to help me write a guidance function for a diffusion-based planner. 
1. The guidance function is used to steer the sampling process toward desired outcomes during the reverse diffusion process. 
2. The guidance function should be differentiable, which computes a scalar reward indicating how well each intermediate trajectory aligns with the task objectives.

In manipulation tasks involving interaction with an object, such as rotating a pen, note that we cannot directly control the object's state. Thus, the guidance function should consider a two-phase approach:
[optional] Phase 1 (Pre-Interaction Phase): The guidance function should focus solely on guiding the hand's state to align with the object's handle or interaction point.
Phase 2 (Post-Interaction Phase): Once the hand is in contact with the object, the guidance function should aim to move the object towards achieving the task goal. During this phase, the guidance function typically include the following components (some part is optional, so only include them if really necessary):
1. difference between the current state of the object and its goal state
2. dynamics constraints to ensure the interactions between the hand and the object are physically plausible
3. regularization of the object's state change (e.g., encourage the hand joint movement to enhance interaction with the object).
4. [optional] extra constraint of the target object, which is often implied by the task instruction
5. [optional] extra constraint of the robot, which is often implied by the task instruction
...

/*\textbf{Environment Description:}*/
class BaseEnv(gym.Env):
    self.hand : AdroitHand     # The Adroit Shadow Hand used in the environment
    self.pen : Pen             # The Pen object in the environment
    self.target : Target       # The target orientation for the pen
    self.dt : float            # The time between two actions, in seconds

    def get_obs(self) -> np.ndarray[(36,)]:
        # Returns the observation vector
        obs = np.concatenate([
            self.hand.get_joint_positions(),                      # Indices 0-23
            self.pen.get_qpos()                                   # Indices 24-29
            self.pen.get_relative_rotation(),                     # Indices 30-32
            self.target.get_relative_rotation(),                  # Indices 33-35
        ])
        return obs

class AdroitHand:
    self.wrist : Wrist         # The wrist component of the hand
    self.fingers : Fingers     # The fingers of the hand
    self.palm : Palm           # The palm of the hand

    def get_joint_positions(self) -> np.ndarray[(24,)]:
        # Returns the angular positions of all joints in the hand
        return np.array([
            self.wrist.wrist_joint_1.angle,       # Index 0: WRJ1
            self.wrist.wrist_joint_0.angle,       # Index 1: WRJ0
            # Finger joints
            self.fingers.ffj3.angle,              # Index 2: FFJ3
            self.fingers.ffj2.angle,              # Index 3: FFJ2
            self.fingers.ffj1.angle,              # Index 4: FFJ1
            self.fingers.ffj0.angle,              # Index 5: FFJ0
            self.fingers.mfj3.angle,              # Index 6: MFJ3
            self.fingers.mfj2.angle,              # Index 7: MFJ2
            self.fingers.mfj1.angle,              # Index 8: MFJ1
            self.fingers.mfj0.angle,              # Index 9: MFJ0
            self.fingers.rfj3.angle,              # Index 10: RFJ3
            self.fingers.rfj2.angle,              # Index 11: RFJ2
            self.fingers.rfj1.angle,              # Index 12: RFJ1
            self.fingers.rfj0.angle,              # Index 13: RFJ0
            self.fingers.lfj4.angle,              # Index 14: LFJ4
            self.fingers.lfj3.angle,              # Index 15: LFJ3
            self.fingers.lfj2.angle,              # Index 16: LFJ2
            self.fingers.lfj1.angle,              # Index 17: LFJ1
            self.fingers.lfj0.angle,              # Index 18: LFJ0
            self.fingers.thj4.angle,              # Index 19: THJ4
            self.fingers.thj3.angle,              # Index 20: THJ3
            self.fingers.thj2.angle,              # Index 21: THJ2
            self.fingers.thj1.angle,              # Index 22: THJ1
            self.fingers.thj0.angle               # Index 23: THJ0
        ])

class Wrist:
    self.wrist_joint_1 : HingeJoint  # WRJ1
    self.wrist_joint_0 : HingeJoint  # WRJ0

class Fingers:
    # Forefinger joints
    self.ffj3 : HingeJoint  # FFJ3
    self.ffj2 : HingeJoint  # FFJ2
    self.ffj1 : HingeJoint  # FFJ1
    self.ffj0 : HingeJoint  # FFJ0

    # Middle finger joints
    self.mfj3 : HingeJoint  # MFJ3
    self.mfj2 : HingeJoint  # MFJ2
    self.mfj1 : HingeJoint  # MFJ1
    self.mfj0 : HingeJoint  # MFJ0

    # Ring finger joints
    self.rfj3 : HingeJoint  # RFJ3
    self.rfj2 : HingeJoint  # RFJ2
    self.rfj1 : HingeJoint  # RFJ1
    self.rfj0 : HingeJoint  # RFJ0

    # Little finger joints
    self.lfj4 : HingeJoint  # LFJ4
    self.lfj3 : HingeJoint  # LFJ3
    self.lfj2 : HingeJoint  # LFJ2
    self.lfj1 : HingeJoint  # LFJ1
    self.lfj0 : HingeJoint  # LFJ0

    # Thumb joints
    self.thj4 : HingeJoint  # THJ4
    self.thj3 : HingeJoint  # THJ3
    self.thj2 : HingeJoint  # THJ2
    self.thj1 : HingeJoint  # THJ1
    self.thj0 : HingeJoint  # THJ0

class Palm:
    self.pose : ObjectPose         # The 3D position and orientation of the palm

    def get_position(self) -> np.ndarray[(3,)]:
        # Returns the position of the palm in world coordinates
        return self.pose.position

class Pen:
    self.pose : ObjectPose         # The 3D position and orientation of the pen
    self.qpos : np.ndarray[(6,)]   # The qpos values of the pen's joints

    def get_position(self) -> np.ndarray[(3,)]:
        # Returns the position of the pen in world coordinates
        return self.pose.position

    def get_relative_rotation(self) -> np.ndarray[(3,)]:
        # Returns the relative rotation of the pen
        return self.pose.orientation

    def get_position_to_target(self, target: Target) -> np.ndarray[(3,)]:
        # Returns the position vector from the pen to the target
        return target.pose.position - self.pose.position

    def get_rotation_to_target(self, target: Target) -> np.ndarray[(3,)]:
        # Returns the rotation vector from the pen to the target
        return target.pose.orientation - self.pose.orientation
        
    def get_qpos(self) -> np.ndarray[(6,)]:
        # Returns the qpos values of the pen's joints
        return self.qpos

class Target:
    self.pose : ObjectPose         # The 3D position

/*\textbf{Additional knowledge:}*/
1. All angles are expressed in radians.
2. The input `normed_obs` is a tensor with shape (B, H, obs_dim), `normed_actions` is a tensor with shape (B, H, act_dim), where B is the batch size, H is the horizon length. The normed_obs is gotten from `normed_obs = get_obs()`.
3. If you need to match the observations or actions to some explicit value and if not without_normalizer, you should unnormalize them using `self.unnormalize(normed_obs, is_obs=True)`.
4. If `dyn_model` is provided, please call `self.cal_dyn_reward(state=normed_obs, action=normed_actions)` to calculates the reward for dynamics inconsistency (a scalar value) between generated states and actions. Only consider it in phase 2. Pay attention the input should be normed_obs and normed_actions before unnormalizing them.
5. Use L2 distance via `torch.norm(,p=2)` to calculate all the difference instead of mse loss or `torch.abs`. For terms where the last dimension is 1 (such as angles), we should use torch.squeeze to remove that dimension before calculating the norm at dimension 1, rather than dimension 2.

You are allowed to use any existing Python package if applicable, but only use them when absolutely necessary. Please import the required packages at the beginning of the function. 

/*\textbf{I want it to fulfill the following task: \textcolor{red}{\{"Write a guidance function for a diffusion-based planner that helps the Adroit Shadow Hand rotate the pen to the desired target orientation."\}}}*/
1. Please think step by step and explain what it means in the context of this environment;
2. Then write a differentiable guidance function that guides the planner to generate actions smoothly based on the current normed state and action, with the function prototype as `def guidance_fn(self, normed_obs, normed_actions, dyn_model=None, without_normalizer=False, desired_pen=None)`. The function should return the `reward` as a torch.Tensor of shape `(B,)`;
3. All the reward including the goal achieving reward should be across all horizon steps. For some term, use `torch.mean()` to accumulate reward over the horizon.
4. Use input `desired_pen` as the target rotation, but you should reshape it by `target_rotation = desired_pen[..., -3:].reshape(batch_size, 1, 3).repeat(1, horizon, 1)`. You should first normalize the direction vector and then use inner product to calculate the similarity between two orientations.
5. Don't directly use actions to penalize the reward, but you can use the difference between the current and previous hand joint states to penalize the reward. You encourage the hand joint movement to enhance interaction with the object.
6. Use `self.scaling_factors` as an empty dictionary by default. If the scaling factor for any reward component does not exist, initialize it adaptively to make that first reward term in batch approximately 1 initially, except for the the dynamics reward (make it 2.). 
7. Take care of variables' type, never use functions or variables not provided. Ensure that all operations are compatible with PyTorch tensors and the function is differentiable. Do not use any absolute value operation and inplace operations, e.g. `x += 1`, `x[0] = 1`, using `x = x + 1` instead. 
8. Pay attention to the physical meaning of each dimension in the observation and action data as explained in the environment description above.
9. When you writing code, you can also add some comments as your thought, like this:
```
# Here unnormalize the observations if a normalizer is provided
# Here use `torch.norm` to compute the L2 distance between the current and target angles for the door hinge
```

/*\textbf{Few-shot hint:}*/
1. Ensure that the guidance function uses soft interpolation for targets, e.g., smoothly guiding the pen orientation towards soft goals over the trajectory horizon like `interpolated_angle = (1 - alpha) * current_obj_orien + alpha * desired_orien`. If use soft goals, don't calculate another hard goal reward.
2. No smoothness reward for the pen movement. Only consider the smoothness of the hand joint movement.
\end{promptcode}

\vspace{3pt}
\subsection{Hand Hammer Task Prompt Example}
\begin{promptcode}%[Hand Hammer Task Prompt Example]
You are an expert in robotics, diffusion model, reinforcement learning, and code generation.
We are going to use an Adroit Shadow Hand to complete given tasks. The action space of the robot is a normalized `Box(-1.0, 1.0, (28,), float32)`. 

Now I want you to help me write a guidance function for a diffusion-based planner. 
1. The guidance function is used to steer the sampling process toward desired outcomes during the reverse diffusion process. 
2. The guidance function should be differentiable, which computes a scalar reward indicating how well each intermediate trajectory aligns with the task objectives.

In manipulation tasks involving interaction with an object, such as opening a door, hammer striking, note that we cannot directly control the object's state. Thus, the guidance function should consider a two-phase approach:
Phase 1 (Pre-Interaction Phase): The guidance function should focus solely on guiding the hand's state to align with the object's handle or interaction point.
Phase 2 (Post-Interaction Phase): Once the hand is in contact with the object, the guidance function should aim to move the object towards achieving the task goal. During this phase, the guidance function typically include the following components (some part is optional, so only include them if really necessary):
1. difference between the current state of the object and its goal state
2. dynamics constraints to ensure the interactions between the hand and the object are physically plausible
3. regularization of the object's state change (e.g., limiting the hinge state change of a door to avoid abrupt movements).
4. [optional] extra constraint of the target object, which is often implied by the task instruction
5. [optional] extra constraint of the robot, which is often implied by the task instruction
...

/*\textbf{Environment Description:}*/
class BaseEnv(gym.Env):
    self.hand : AdroitHand     # The Adroit Shadow Hand used in the environment
    self.hammer : Hammer       # The Hammer object in the environment
    self.nail : Nail          # The Nail object in the environment
    self.dt : float           # The time between two actions, in seconds

    def get_obs(self) -> np.ndarray[(46,)]:
        # Returns the observation vector
        obs = np.concatenate([
            self.hand.get_joint_positions(),                      # Indices 0-25
            [self.nail.insertion_displacement],                   # Index 26
            self.hammer.get_qpos(),                               # Indices 27-32
            self.hand.palm.get_position(),                        # Indices 33-35
            self.hammer.get_position(),                           # Indices 36-38
            self.hammer.get_orientation(),                        # Indices 39-41
            self.nail.get_position(),                             # Indices 42-44
            [self.nail.force]                                     # Index 45
        ])
        return obs

class AdroitHand:
    self.arm : Arm             # The arm component of the hand
    self.wrist : Wrist         # The wrist component of the hand
    self.fingers : Fingers     # The fingers of the hand
    self.palm : Palm           # The palm of the hand

    def get_joint_positions(self) -> np.ndarray[(26,)]:
        # Returns the angular positions of all joints in the hand and arm
        return np.array([
            self.arm.rotation_x.angle,            # Index 0: ARRx
            self.arm.rotation_y.angle,            # Index 1: ARRy
            self.wrist.wrist_joint_1.angle,       # Index 2: WRJ1
            self.wrist.wrist_joint_0.angle,       # Index 3: WRJ0
            # Finger joints
            self.fingers.ffj3.angle,              # Index 4: FFJ3
            self.fingers.ffj2.angle,              # Index 5: FFJ2
            self.fingers.ffj1.angle,              # Index 6: FFJ1
            self.fingers.ffj0.angle,              # Index 7: FFJ0
            self.fingers.mfj3.angle,              # Index 8: MFJ3
            self.fingers.mfj2.angle,              # Index 9: MFJ2
            self.fingers.mfj1.angle,              # Index 10: MFJ1
            self.fingers.mfj0.angle,              # Index 11: MFJ0
            self.fingers.rfj3.angle,              # Index 12: RFJ3
            self.fingers.rfj2.angle,              # Index 13: RFJ2
            self.fingers.rfj1.angle,              # Index 14: RFJ1
            self.fingers.rfj0.angle,              # Index 15: RFJ0
            self.fingers.lfj4.angle,              # Index 16: LFJ4
            self.fingers.lfj3.angle,              # Index 17: LFJ3
            self.fingers.lfj2.angle,              # Index 18: LFJ2
            self.fingers.lfj1.angle,              # Index 19: LFJ1
            self.fingers.lfj0.angle,              # Index 20: LFJ0
            self.fingers.thj4.angle,              # Index 21: THJ4
            self.fingers.thj3.angle,              # Index 22: THJ3
            self.fingers.thj2.angle,              # Index 23: THJ2
            self.fingers.thj1.angle,              # Index 24: THJ1
            self.fingers.thj0.angle               # Index 25: THJ0
        ])

class Hammer:
    self.pose : ObjectPose         # The 3D position and orientation of the hammer
    self.velocity : ObjectVelocity # Linear and angular velocities of the hammer
    self.OBJTx : SlideJoint        # The slide joint along the x-axis
    self.OBJTy : SlideJoint        # The slide joint along the y-axis
    self.OBJTz : SlideJoint        # The slide joint along the z-axis
    self.OBJRx : RevoluteJoint     # The revolute joint around the x-axis
    self.OBJRy : RevoluteJoint     # The revolute joint around the y-axis
    self.OBJRz : RevoluteJoint     # The revolute joint around the z-axis

    def get_position(self) -> np.ndarray[(3,)]:
        # Returns the position of the hammer's center of mass in world coordinates
        return self.pose.position

    def get_orientation(self) -> np.ndarray[(3,)]:
        # Returns the relative rotation of the hammer with respect to x,y,z axes
        return self.pose.get_euler_angles()

    def get_qpos(self) -> np.ndarray[(6,)]:
        # Returns the joint positions of the hammer
        return np.array([self.OBJTx.position, self.OBJTy.position, self.OBJTz.position,
                            self.OBJRx.angle, self.OBJRy.angle, self.OBJRz.angle])

class Nail:
    self.pose : ObjectPose            # The 3D position of the nail
    self.insertion_displacement : float # Current insertion depth of the nail
    self.force : float                # Linear force exerted on the nail head

    def get_position(self) -> np.ndarray[(3,)]:
        # Returns the position of the nail in world coordinates
        return self.pose.position

class ObjectVelocity:
    self.linear : np.ndarray[(3,)]    # Linear velocity in x,y,z
    self.angular : np.ndarray[(3,)]   # Angular velocity around x,y,z axes

class ObjectPose:
    self.position : np.ndarray[(3,)]    # 3D position in world coordinates
    self.orientation : np.ndarray[(4,)]  # Quaternion orientation (w, x, y, z)

    def get_euler_angles(self) -> np.ndarray[(3,)]:
        # Returns the orientation as Euler angles (roll, pitch, yaw)
        return quaternion_to_euler(self.orientation)

Observation Index Mapping:
Index 0-25: Angular positions of the hand joints (in radians);
Index 26: Insertion displacement of nail (in meters) range from -0.01 to 0.09;
Index 27-32: Qpos of the hammer joints (in meters and radians);
Index 33-35: Position of the center of the palm in x,y,z (in meters);
Index 36-38: Position of the hammer's center of mass in x,y,z (in meters);
Index 39-41: Relative rotation of hammer's center of mass w.r.t x,y,z axes (in radians);
Index 42-44: Position of the nail in x,y,z (in meters);
Index 45: Linear force exerted on the head of the nail (in Newtons) range from -1.0 to 1.0.

/*\textbf{Additional knowledge:}*/
1. All angles are expressed in radians.
2. The input `normed_obs` is a tensor with shape (B, H, obs_dim), `normed_actions` is a tensor with shape (B, H, act_dim), where B is the batch size, H is the horizon length. The normed_obs is gotten from `normed_obs = get_obs()`.
3. If you need to match the observations or actions to some explicit value and if not without_normalizer, you should unnormalize them using `self.unnormalize(normed_obs, is_obs=True)`.
4. If `dyn_model` is provided, please call `self.cal_dyn_reward(state=normed_obs, action=normed_actions)` to calculates the reward for dynamics inconsistency (a scalar value) between generated states and actions. Only consider it in phase 2. Pay attention the input should be normed_obs and normed_actions before unnormalizing them.
5. Use L2 distance via `torch.norm(,p=2)` to calculate all the difference instead of mse loss or `torch.abs`.
6. The transition between Phase 1 and Phase 2 by using a grasp mask to determine if the hand has successfully grasped the object. Use a condition like `mask = torch.norm(palm_pos[:, 0, :] - handle_pos[:, 0, :], p=2, dim=1) < 0.1` to switch from guiding only the hand to guiding both the hand and the object.

You are allowed to use any existing Python package if applicable, but only use them when absolutely necessary. Please import the required packages at the beginning of the function.

/*\textbf{I want it to fulfill the following task: \textcolor{red}{\{"Write a guidance function for a diffusion-based planner that helps the Adroit Shadow Hand grasp the hammer and only drive half nail into the board."\}}}*/
1. Please think step by step and explain what it means in the context of this environment;
2. Then write a differentiable guidance function that guides the planner to generate actions smoothly based on the current normed state and action, with the function prototype as `def guidance_fn(self, normed_obs, normed_actions, dyn_model=None, without_normalizer=False)`. The function should return the `reward` as a torch.Tensor of shape `(B,)`;
3. Make sure the guidance aligns with the two phases: In Phase 1, only calculate a pre-grasp reward to guide the hand closer to the object. In Phase 2, guide both the object toward the final task goal. Ensure object velocity constraints are applied to regulate object state changes.
4. All the reward including the goal achieving reward should be across all horizon steps. For some term, use `torch.mean()` to accumulate reward over the horizon. For terms where the last dimension is 1 (such as angles), we should use torch.squeeze to remove that dimension before calculating the norm at dimension 1, rather than dimension 2.
5. Take care of variables' type, never use functions or variables not provided. Ensure that all operations are compatible with PyTorch tensors and the function is differentiable. Do not use any absolute value operation and inplace operations, e.g. `x += 1`, `x[0] = 1`, using `x = x + 1` instead. 
6. Pay attention to the physical meaning of each dimension in the observation and action data as explained in the environment description above.
7. When you writing code, you can also add some comments as your thought, like this:
```
# Here unnormalize the observations if a normalizer is provided
# Here use `torch.norm` to compute the L2 distance between the current and target angles for the door hinge
# Here cauculate the grasp mask for the pre-interaction phase
```

/*\textbf{Few-shot hint:}*/
1. Use nail insertion displacement (range from -0.01 No Drive to 0.09 Full Drive) as a measure of nail insertion degree.
2. Constraint both the hammer qpos changes and the hammer position & rotation changes (pos+rot as 1 item) to avoid abrupt movements.
3. Use `self.scaling_factors` as an empty dictionary by default. If the scaling factor for any reward component does not exist, initialize it adaptively to make that first reward term in batch approximately 6 initially, except for the hammer qpos change constraint reward (make the reward 12) and the dynamics reward (make it 0.3).
\end{promptcode}

Then we can replace the task instruction marked by red with specific requirements according to different tasks.
This structured prompting approach enables LLMs to generate physically consistent and task-appropriate guidance functions for goal-adaptive dexterous manipulation.

\section{Samples of Generated Guidance Function}
\subsection{Sample of Guidance Function on Hand Door Task}
\begin{pythoncode}
def guidance_fn(self, normed_obs, normed_actions, dyn_model=None, without_normalizer=False):
    import math
    """
    Guidance function for door opening task with the Adroit hand.
    Args:
        normed_obs: Normalized observations tensor of shape (B, H, obs_dim)
        normed_actions: Normalized actions tensor of shape (B, H, act_dim)
        dyn_model: Optional dynamics model for consistency checking
        without_normalizer: Boolean indicating if normalization should be skipped
    Returns:
        reward: Tensor of shape (B,) containing the guidance rewards
    """
    batch_size, horizon, _ = normed_obs.shape

    # Get the unnormalized observations if normalizer is provided
    obs = normed_obs if without_normalizer else self.unnormalize(normed_obs, is_obs=True)

    # Extract relevant state information
    palm_pos = obs[:, :, 30:33]  # Palm position
    handle_pos = obs[:, :, 33:36]  # Door handle position
    door_hinge_angle = obs[:, :, 28:29]  # Door hinge angle

    # Define target door angle (30 degrees = math.pi/6 radians)
    target_door_angle = torch.full_like(door_hinge_angle, math.pi / 6)

    # Calculate grasp mask for phase transition
    # Use distance between palm and handle to determine if hand is close enough
    grasp_threshold = 0.1
    grasp_dist = torch.norm(palm_pos - handle_pos, p=2, dim=2)
    grasp_mask = (grasp_dist < grasp_threshold).float()

    # Phase 1: Pre-interaction rewards
    # Calculate reaching reward - guide hand to door handle
    reaching_dist = torch.norm(palm_pos - handle_pos, p=2, dim=2)
    reaching_reward = -torch.mean(reaching_dist, dim=1)  # Average over horizon

    # Scale reaching reward if not already scaled
    if 'reaching' not in self.scaling_factors:
        # Initialize scaling to make initial reward approximately 12
        with torch.no_grad():
            initial_reaching = -reaching_reward[0].item()
            self.scaling_factors['reaching'] = 12.0 / (initial_reaching + 1e-8)
    reaching_reward = reaching_reward * self.scaling_factors['reaching']

    # Phase 2: Post-interaction rewards
    # Calculate door angle reward
    door_angle_diff = torch.norm(door_hinge_angle - target_door_angle, p=2, dim=2)
    door_reward = -torch.mean(door_angle_diff, dim=1)  # Average over horizon

    # Scale door reward if not already scaled
    if 'door' not in self.scaling_factors:
        # Initialize scaling to make initial reward approximately 30
        with torch.no_grad():
            initial_door = -door_reward[0].item()
            self.scaling_factors['door'] = 30.0 / (initial_door + 1e-8)
    door_reward = door_reward * self.scaling_factors['door']

    # Calculate velocity constraint reward to ensure smooth door movement
    door_velocity = (door_hinge_angle[:, 1:, 0] - door_hinge_angle[:, :-1, 0]) / self.dt
    velocity_reward = -torch.norm(door_velocity, p=2, dim=1)

    # Scale velocity reward if not already scaled
    if 'velocity' not in self.scaling_factors:
        with torch.no_grad():
            initial_velocity = -velocity_reward[0].item()
            self.scaling_factors['velocity'] = 12.0 / (initial_velocity + 1e-8)
    velocity_reward = velocity_reward * self.scaling_factors['velocity']

    # Add dynamics consistency reward if model is provided
    dyn_reward = torch.zeros_like(reaching_reward)
    if dyn_model is not None:
        dyn_reward = self.cal_dyn_reward(state=normed_obs, action=normed_actions)
        # Scale dynamics reward if not already scaled
        if 'dynamics' not in self.scaling_factors:
            with torch.no_grad():
                initial_dyn = dyn_reward[0].item()
                self.scaling_factors['dynamics'] = 1.2 / (initial_dyn + 1e-8)
        dyn_reward = dyn_reward * self.scaling_factors['dynamics']

    # Combine rewards using the grasp mask
    # Pre-interaction phase: only reaching reward
    # Post-interaction phase: door reward + velocity reward + dynamics reward
    total_reward = (1 - grasp_mask[:, 0]) * reaching_reward + \
                   grasp_mask[:, 0] * (door_reward + velocity_reward + dyn_reward)

    return total_reward
\end{pythoncode}

\vspace{3pt}
\subsection{Sample of Guidance Function on Hand Pen Task}
\begin{pythoncode}
def guidance_fn(self, normed_obs, normed_actions, dyn_model=None, without_normalizer=False, desired_pen=None):
    # Get batch size and horizon length
    batch_size, horizon, obs_dim = normed_obs.shape
    
    # Process target rotation
    target_rotation = desired_pen[..., -3:].reshape(batch_size, 1, 3).repeat(1, horizon, 1)
    
    # Initialize reward components
    reward = torch.zeros(batch_size, device=normed_obs.device)
    
    # Unnormalize observations if normalizer is provided
    if not without_normalizer:
        obs = self.unnormalize(normed_obs, is_obs=True)
    else:
        obs = normed_obs

    # Extract relevant states
    hand_joints = obs[..., :24]  # Hand joint positions
    pen_rotation = obs[..., 30:33]  # Pen rotation
    
    # 1. Orientation alignment reward using soft interpolation
    # Normalize vectors before computing similarity
    pen_rotation_norm = pen_rotation / (torch.norm(pen_rotation, p=2, dim=-1, keepdim=True) + 1e-6)
    target_rotation_norm = target_rotation / (torch.norm(target_rotation, p=2, dim=-1, keepdim=True) + 1e-6)
    
    # Compute similarity using dot product (higher means more aligned)
    orientation_similarity = torch.sum(pen_rotation_norm * target_rotation_norm, dim=-1)
    orientation_reward = torch.mean(orientation_similarity, dim=1)
    
    # Initialize scaling factor for orientation reward if not exists
    if 'orientation' not in self.scaling_factors:
        self.scaling_factors['orientation'] = 1.0 / (orientation_reward[0].abs().item() + 1e-6)
    
    reward = reward + self.scaling_factors['orientation'] * orientation_reward

    # 2. Hand joint movement smoothness reward
    # Calculate joint position differences between consecutive timesteps
    joint_diffs = hand_joints[:, 1:] - hand_joints[:, :-1]
    smoothness_penalty = torch.norm(joint_diffs, p=2, dim=-1)  # Shape: (batch_size, horizon-1)
    smoothness_reward = -torch.mean(smoothness_penalty, dim=1)  # Average over horizon
    
    # Initialize scaling factor for smoothness reward if not exists
    if 'smoothness' not in self.scaling_factors:
        self.scaling_factors['smoothness'] = 1.0 / (smoothness_reward[0].abs().item() + 1e-6)
    
    reward = reward + self.scaling_factors['smoothness'] * smoothness_reward

    # 3. Dynamic consistency reward (if model provided)
    if dyn_model is not None:
        dyn_reward = self.cal_dyn_reward(state=normed_obs, action=normed_actions)
        
        # Initialize scaling factor for dynamics reward if not exists
        if 'dynamics' not in self.scaling_factors:
            self.scaling_factors['dynamics'] = 2.0 / (dyn_reward[0].abs().item() + 1e-6)
        
        reward = reward + self.scaling_factors['dynamics'] * dyn_reward

    return reward
\end{pythoncode}

\vspace{3pt}
\subsection{Sample of Guidance Function on Hand Hammer Task}
\begin{pythoncode}
def guidance_fn(self, normed_obs, normed_actions, dyn_model=None, without_normalizer=False, tool_pos=None):
    """
    Guidance function for hammer-nail task with Adroit hand.
    Args:
        normed_obs: Normalized observations, shape (B, H, obs_dim)
        normed_actions: Normalized actions, shape (B, H, act_dim)
        dyn_model: Optional dynamics model for consistency checking
        without_normalizer: Boolean indicating if normalization should be skipped
    Returns:
        reward: Total reward tensor of shape (B,)
    """
    batch_size = normed_obs.shape[0]
    horizon_len = normed_obs.shape[1]
    device = normed_obs.device

    # Get unnormalized observations if normalizer is provided
    obs = normed_obs if without_normalizer else self.unnormalize(normed_obs, is_obs=True)

    # Extract relevant observations across all timesteps
    palm_pos = obs[:, :, 33:36]  # Hand palm position
    hammer_pos = obs[:, :, 36:39]  # Hammer position
    nail_pos = obs[:, :, 42:45]  # Nail position
    nail_insertion = obs[:, :, 26]  # Nail insertion depth, keep dim for proper broadcasting
    tool_pos = tool_pos[:, None, :].repeat(1, horizon_len, 1)

    # Calculate grasp mask based on distance between palm and hammer
    # Use first timestep to determine if hand has grasped hammer
    grasp_threshold = 0.1
    grasp_mask = torch.norm(palm_pos[:, 0, :] - hammer_pos[:, 0, :], p=2, dim=1) < grasp_threshold

    # Initialize total reward
    total_reward = torch.zeros(batch_size, device=device)

    # Phase 1: Pre-interaction guidance (hand approaching hammer)
    pre_grasp_reward = -torch.mean(
        torch.norm(palm_pos - hammer_pos, p=2, dim=2),
        dim=1
    )

    # Adaptive scaling for pre-grasp reward
    if 'pre_grasp' not in self.scaling_factors:
        self.scaling_factors['pre_grasp'] = 6.0 / (torch.abs(pre_grasp_reward[0]) + 1e-6)

    total_reward = total_reward + self.scaling_factors['pre_grasp'] * pre_grasp_reward

    # Phase 2: Post-interaction guidance (hammer control and nail insertion)
    # Only apply if hand has grasped hammer
    if torch.any(grasp_mask):
        contact_mask = torch.norm(tool_pos - nail_pos, p=2, dim=2) < 0.1
        # Target nail insertion (halfway = 0.04m)
        target_insertion = 0.04 * torch.ones_like(nail_insertion)
        insertion_reward = \
            -torch.norm(nail_insertion - target_insertion, p=2, dim=1) #* contact_mask[:, 0]

        # Adaptive scaling for insertion reward
        if 'insertion' not in self.scaling_factors:
            self.scaling_factors['insertion'] = 6.0 / (torch.abs(insertion_reward[0]) + 1e-6)

        # Constraint on hammer position changes (smooth movement)
        hammer_joint_pos_changes = torch.norm(
            obs[:, 1:, 27:33] - obs[:, :-1, 27:33],
            p=2, dim=2
        )
        hammer_joint_reward = -torch.mean(hammer_joint_pos_changes, dim=1)

        # Adaptive scaling for nail movement constraint
        if 'hammer_joint' not in self.scaling_factors:
            self.scaling_factors['hammer_joint'] = 6.0 / (torch.abs(hammer_joint_reward[0]) + 1e-6)

        # Constraint on hammer position changes (smooth movement)
        hammer_pos_changes = torch.norm(
            hammer_pos[:, 1:, :] - hammer_pos[:, :-1, :],
            p=2, dim=2
        )
        hammer_movement_reward = -torch.mean(hammer_pos_changes, dim=1)

        # Adaptive scaling for hammer movement constraint
        if 'hammer_movement' not in self.scaling_factors:
            self.scaling_factors['hammer_movement'] = 12.0 / (torch.abs(hammer_movement_reward[0]) + 1e-6)  # 100.

        # Add dynamics consistency reward if model provided
        if dyn_model is not None:
            dyn_reward = -self.cal_dyn_reward(state=normed_obs, action=normed_actions)

            # Adaptive scaling for dynamics reward
            if 'dynamics' not in self.scaling_factors:
                self.scaling_factors['dynamics'] = 0.3 / (torch.abs(dyn_reward[0]) + 1e-6)

            # Apply dynamics reward only to grasped trajectories
            total_reward = total_reward + self.scaling_factors['dynamics'] * dyn_reward * grasp_mask.float()

        # Add all Phase 2 rewards
        phase2_reward = (self.scaling_factors['insertion'] * insertion_reward +
                        self.scaling_factors['hammer_joint'] * hammer_joint_reward +
                         self.scaling_factors['hammer_movement'] * hammer_movement_reward)

        # Apply Phase 2 rewards only to grasped trajectories
        total_reward = total_reward + phase2_reward * grasp_mask.float()

    return total_reward
\end{pythoncode}

\twocolumn

\end{document}